\documentclass[10pt]{article}
\usepackage[margin=1in]{geometry}
\usepackage[round,compress]{natbib}
\usepackage{parskip}

\usepackage[utf8]{inputenc} %
\usepackage[T1]{fontenc}    %

\usepackage{subcaption}

\usepackage{amsmath}
\usepackage{arydshln}
\usepackage{amsfonts,amscd,amssymb,bm,bbm,mathrsfs}
\usepackage{nicefrac}       %
\usepackage{amsthm}
\usepackage{mathtools}

\newcommand\labelAndRemember[2]
  {\expandafter\gdef\csname labeled:#1\endcsname{#2}%
   \label{#1}#2}
\newcommand\recallLabel[1]
   {\csname labeled:#1\endcsname\tag{\ref{#1}}}
\newcommand\labelr[2]
  {\expandafter\gdef\csname labeled:#1\endcsname{#2}%
   \label{#1}#2}
\newcommand\recall[1]
   {\csname labeled:#1\endcsname}

\usepackage[algo2e,linesnumbered,vlined,ruled]{algorithm2e}
\usepackage{algorithm}
\usepackage[noend]{algorithmic}

\usepackage{url}
\usepackage[breaklinks=true]{hyperref}
\usepackage{breakcites}
\hypersetup{
  colorlinks,
  citecolor=blue!70!black,
  linkcolor=blue!70!black,
}
\usepackage[capitalise]{cleveref}
\usepackage[titletoc,title]{appendix}
\usepackage{cancel}
\usepackage{enumerate}
\usepackage{enumitem}
\usepackage{comment}
\crefformat{equation}{#2(#1)#3}
\Crefformat{equation}{#2(#1)#3}
\crefformat{figure}{Figure #2#1#3}
\Crefformat{figure}{Figure #2#1#3}

\usepackage{booktabs}       %
\usepackage{multirow}
\usepackage{multicol}
\usepackage{makecell}
\usepackage{array}
\newcolumntype{H}{>{\setbox0=\hbox\bgroup}c<{\egroup}@{}}
\newcolumntype{Z}{>{\setbox0=\hbox\bgroup}c<{\egroup}@{\hspace*{-\tabcolsep}}}

\usepackage[normalem]{ulem}
\usepackage{exscale}
\usepackage{tikz}
\usepackage{float}
\usepackage{graphicx,graphics}
\graphicspath{ {./fig/} }
\allowdisplaybreaks

\newtheorem{theorem}{Theorem}
\newtheorem*{theorem*}{Theorem}
\newtheorem{lemma}[theorem]{Lemma}

\newtheorem*{remark*}{Remark}
\newtheorem*{lemma*}{Lemma}

\newtheorem{proposition}[theorem]{Proposition}

\renewcommand{\theassumption}{\Alph{assumption}}

\newenvironment{proof-sketch}{\noindent{\bf Proof Sketch}
  \hspace*{1em}}{\qed\bigskip\\}
\newenvironment{proof-idea}{\noindent{\bf Proof Idea}
  \hspace*{1em}}{\qed\bigskip\\}
\newenvironment{proof-of}[1][{}]{\noindent{\bf Proof of \cref{#1}}
  \hspace*{1em}}{\qed\bigskip\\}
\newenvironment{proof-of-lemma}[1][{}]{\noindent{\bf Proof of Lemma {#1}}
  \hspace*{1em}}{\qed\bigskip\\}
\newenvironment{proof-of-proposition}[1][{}]{\noindent{\bf
    Proof of Proposition {#1}}
  \hspace*{1em}}{\qed\bigskip\\}
\newenvironment{proof-of-theorem}[1][{}]{\noindent{\bf Proof of Theorem {#1}}
  \hspace*{1em}}{\qed\bigskip\\}
\newenvironment{inner-proof}{\noindent{\bf Proof}\hspace{1em}}{
  $\bigtriangledown$\medskip\\}
\newenvironment{proof-attempt}{\noindent{\bf Proof Attempt}
  \hspace*{1em}}{\qed\bigskip\\}

\renewcommand{\hat}{\widehat}
\renewcommand{\bar}{\overline}
\renewcommand{\epsilon}{\varepsilon}

\usepackage{pifont}

\newcommand{\eps}{\varepsilon}

\newcounter{cnt}
\foreach \num in {1,2,...,26}{%
  \stepcounter{cnt}%
  \expandafter\xdef \csname c\Alph{cnt}\endcsname {\noexpand\mathcal{\Alph{cnt}}}%
  \expandafter\xdef \csname b\Alph{cnt}\endcsname {\noexpand\mathbb{\Alph{cnt}}}%
}

\DeclareMathOperator*{\argmin}{arg\,min}

\newcommand{\abs}[1]{\left|#1\right|}
\newcommand{\abss}[1]{|#1|}

\newcommand{\ceil}[1]{\left\lceil #1\right\rceil}

\DeclarePairedDelimiterX{\ddiv}[2]{(}{)}{%
  #1\;\delimsize\|\;#2%
}

\newcommand{\norm}[1]{\left\|{#1}\right\|} %
\newcommand{\ltwo}[1]{\norm{#1}_2} %

\newcommand{\linf}[1]{\norm{#1}_\infty} %
\newcommand{\lfro}[1]{\left\|{#1}\right\|_{\sf Fr}} %
\newcommand{\norms}[1]{\|{#1}\|} %
\newcommand{\ltwos}[1]{\norms{#1}_2} %

\newcommand{\linfs}[1]{\norms{#1}_\infty} %

\renewcommand{\cO}{\mathcal{O}}
\newcommand{\tO}{\widetilde{\cO}}
\newcommand{\wt}{\widetilde}
\newcommand{\indic}[1]{1\{#1\}}

\newcommand{\<}{\left\langle}
\renewcommand{\>}{\right\rangle}

\renewcommand{\bQ}{\mathbf{Q}}

\newenvironment{talign}
 {\align}
 {\endalign}
\newenvironment{talign*}
 {\csname align*\endcsname}
 {\endalign}

\newcommand{\Attn}{{\rm Attn}}

\newcommand{\TF}{{\rm TF}}

\newcommand{\barsig}{\overline{\sigma}}

\newcommand{\Bx}{B_x}
\newcommand{\By}{B_y}
\newcommand{\Bw}{B_w}

\newcommand{\bzero}{{\mathbf 0}}

\newcommand{\hbw}{\hat{\bw}}

\newcommand{\Pin}{\mathsf{P}}

\newcommand{\ltwoinfs}[1]{\|{#1}\|_{2,\infty}}

\newcommand{\gd}{{\rm GD}}

\newcommand{\normal}{\mathsf{N}}

\newcommand{\MLP}{\mathrm{MLP}}

\newcommand{\lammax}{\lambda_{\max}}

\newcommand{\tbh}{\wt{\bh}}
\newcommand{\tbH}{\wt{\bH}}

\newcommand{\MSK}{{\rm MSK}}
\mathchardef\mhyphen="2D

\newcommand{\hy}{\hat{y}}

\newcommand{\softmax}{{\sf softmax}}

\newcommand{\tPhi}{\wt{\Phi}}
\newcommand{\tbp}{\wt{\bp}}
\newcommand{\Dhid}{D_{\rm hid}}

\newcommand{\hbx}{\hat{\bx}}
\newcommand{\hby}{\hat{\by}}
\newcommand{\Li}{\hat{L}_i}
\newcommand{\Lim}{\hat{L}_{i-1}}
\newcommand{\Lilam}{\hat{L}_i^{\lambda}}
\newcommand{\Limlam}{\hat{L}_{i-1}^{\lambda}}
\newcommand{\hbwilam}{\hat{\bw}_i^\lambda}
\newcommand{\hbWilam}{\hat{\bW}_i^\lambda}
\newcommand{\hbwiphilam}{\hat{\bw}_i^{\Phi^\star,\lambda}}
\newcommand{\hbWiphilam}{\hat{\bW}_i^{\Phi^\star,\lambda}}
\newcommand{\hyiphilam}{\hy_i^{\Phi^\star,\lambda}}

\newcommand{\barbh}{\bar{\bh}}
\newcommand{\barbH}{\bar{\bH}}
\newcommand{\Bphi}{B_\Phi}
\newcommand{\barbp}{\bar{\bp}}
\newcommand{\barbx}{\bar{\bx}}
\newcommand{\barbw}{\bar{\bw}}
\newcommand{\barPhi}{\bar{\Phi}}
\newcommand{\barTF}{\bar{\TF}}

\newcommand{\paren}[1]{{\left( #1 \right)}}
\newcommand{\brac}[1]{{\left[ #1 \right]}}
\newcommand{\set}[1]{{\left\{ #1 \right\}}}
\newcommand{\sets}[1]{{\{ #1 \}}}

\newcommand{\defeq}{\mathrel{\mathop:}=}
\newcommand{\eqdef}{=\mathrel{\mathop:}}

\newcommand{\mat}[1]{\ensuremath{\mathbf{#1}}}

\newcommand{\grad}{\nabla}

\newcommand{\simiid}{\stackrel{\rm iid}{\sim}}

\newcommand{\E}{\mathbb{E}}

\newcommand{\R}{\mathbb{R}}

\newcommand{\B}{\mat{B}}

\newcommand{\probetarget}{g}

\def\bA{{\mathbf A}}
\def\bB{{\mathbf B}}

\def\bH{{\mathbf H}}
\def\bI{{\mathbf I}}
\def\bK{{\mathbf K}}

\def\bQ{{\mathbf Q}}

\def\bV{{\mathbf V}}
\def\bW{{\mathbf W}}
\def\bY{{\mathbf Y}}
\def\bX{{\mathbf X}}

\def\btheta{{\boldsymbol \theta}}

\def\ba{{\mathbf a}}

\def\bh{{\mathbf h}}

\def\bp{{\mathbf p}}

\def\bu{{\mathbf u}}

\def\bw{{\mathbf w}}
\def\bx{{\mathbf x}}
\def\by{{\mathbf y}}
\def\bz{{\mathbf z}}

\title{How Do Transformers Learn In-Context Beyond Simple Functions? A Case Study on Learning with Representations}

\author{
Tianyu Guo\footnotemark[1]\thanks{UC Berkeley. Email: \texttt{\{tianyu\_guo,songmei\}@berkeley.edu}}
\and
Wei Hu\footnotemark[2]\thanks{University of Michigan. Email: \texttt{vvh@umich.edu}}
\and
Song Mei\footnotemark[1]
\and
Huan Wang\footnotemark[3]\thanks{Salesforce AI Research. Email: \texttt{\{huan.wang, cxiong, ssavarese, yu.bai\}@salesforce.com}}
\and
Caiming Xiong\footnotemark[3]
\and
Silvio Savarese\footnotemark[3]
\and
Yu Bai\footnotemark[3]
\and
}

\def\shownotes{0}  %
\ifnum\shownotes=1
\newcommand{\authnote}[2]{{\scriptsize $\ll$\textsf{#1 notes: #2}$\gg$}}
\else
\newcommand{\authnote}[2]{}
\fi

\begin{document}

\maketitle

\begin{abstract}

While large language models based on the transformer architecture have demonstrated remarkable in-context learning (ICL) capabilities, understandings of such capabilities are still in an early stage, where existing theory and mechanistic understanding focus mostly on simple scenarios such as learning simple function classes. This paper takes initial steps on understanding ICL in more complex scenarios, by studying learning with \emph{representations}. Concretely, we construct synthetic in-context learning problems with a compositional structure, where the label depends on the input through a possibly complex but \emph{fixed} representation function, composed with a linear function that \emph{differs} in each instance. By construction, the optimal ICL algorithm first transforms the inputs by the representation function, and then performs linear ICL on top of the transformed dataset. We show theoretically the existence of transformers that approximately implement such algorithms with mild depth and size.  Empirically, we find trained transformers consistently achieve near-optimal ICL performance in this setting, and exhibit the desired dissection where lower layers transforms the dataset and upper layers perform linear ICL. Through extensive probing and a new pasting experiment, we further reveal several mechanisms within the trained transformers, such as concrete copying behaviors on both the inputs and the representations, linear ICL capability of the upper layers alone, and a post-ICL representation selection mechanism in a harder mixture setting. These observed mechanisms align well with our theory and may shed light on how transformers perform ICL in more realistic scenarios.

\end{abstract}

\section{Introduction}

Large language models based on the transformer architecture have demonstrated remarkable in-context learning (ICL) capabilities~\citep{brown2020language}, where they can solve newly encountered tasks when prompted with only a few training examples, without any parameter update to the model. Recent state-of-the-art models further achieve impressive performance in context on sophisticated real-world tasks~\citep{openai2023gpt,bubeck2023sparks,touvron2023llama}. Such remarkable capabilities call for better understandings, which recent work tackles from various angles~\citep{xie2021explanation,chan2022data,razeghi2022impact,min2022rethinking,olsson2022context,wei2023larger}.

A recent surge of work investigates ICL in a theoretically amenable setting where the context consists of real-valued (input, label) pairs generated from a certain function class. They find that transformers can learn many function classes in context, such as linear functions, shallow neural networks, and decision trees~\citep{garg2022can,akyurek2022learning,li2023transformers}, and further studies provide theoretical justification on how transformers can implement and learn various learning algorithms in-context such as ridge regression~\citep{akyurek2022learning}, gradient descent~\citep{von2022transformers,dai2022can,zhang2023trained,ahn2023transformers}, algorithm selection~\citep{bai2023transformers}, and Bayes model averaging~\citep{zhang2023and}, to name a few. Despite the progress, an insufficiency of this line is that the settings and results may not actually resemble ICL in real-world scenarios---For example, ICL in linear function classes are well understood in theory with efficient transformer constructions~\citep{bai2023transformers}, and transformers indeed learn them well empirically~\citep{garg2022can}; however, such linear functions in the raw input may fail to capture real-world scenarios where \emph{prior knowledge} can often aid learning.

This paper takes initial steps towards addressing this by studying ICL in the setting of \emph{learning with representations}, a more complex and perhaps more realistic setting than existing ones. We construct synthetic ICL tasks where labels depend on inputs through a fixed representation function composed with a varying linear function. We instantiate the representation as shallow neural networks (MLPs), and consider both a supervised learning setting (with input-label pairs) and a dynamical systems setting (with inputs only) for the in-context data.  Our contributions can be summarized as follows.
\begin{itemize}[leftmargin=1.5em, topsep=0pt, itemsep=0pt]
\item Theoretically, we construct transformers that implement in-context ridge regression on the representations (which includes the Bayes-optimal algorithm) for both learning settings (\cref{sec:theory}). Our transformer constructions admit mild sizes, and can predict at every token using a decoder architecture, (non-trivially) generalizing existing efficient constructions that predict at the last token only using an encoder architecture. 
\item Empirically, we find that trained small transformers consistently achieve near-optimal ICL risk in both learning settings (\cref{sec:exp} \&~\cref{fig:fig1-risk}).
\item Using linear probing techniques, we identify evidence for various mechanisms in the trained transformers. Our high-level finding is that the lower layers transforms the data by the representation and prepares it into a certain format, and the upper layers perform linear ICL on top of the transformed data (\cref{fig:fig1-probe}), with often a clear dissection between these two modules, consistent with our theory. See~\cref{fig:fig1-illustration} for a pictorial illustration.
\item We further observe several lower-level behaviors using linear probes that align well with our (and existing) theoretical constructions, such as copying (of both the input and the representations) where which tokens are being copied are precisely identifiable (\cref{sec:exp-dynamical-system}), and a post-ICL representation selection mechanism in a harder setting (\cref{sec:mtl-short} \&~\cref{app:mtl}).
\item We perform a new pasting experiment and find that the upper layers within the trained transformer can perform nearly-optimal linear ICL in (nearly-)isolation (\cref{sec:exp-fixed}), which provides stronger evidence that the upper module alone can be a strong linear ICL learner.
\end{itemize}

\begin{figure}[t]
  \centering
  \begin{minipage}{0.37\textwidth}
      \centering
      \subcaption{\small Illustration of our setting and theory}
      \label{fig:fig1-illustration}
      \vspace{.5em}
      \includegraphics[width=\linewidth]{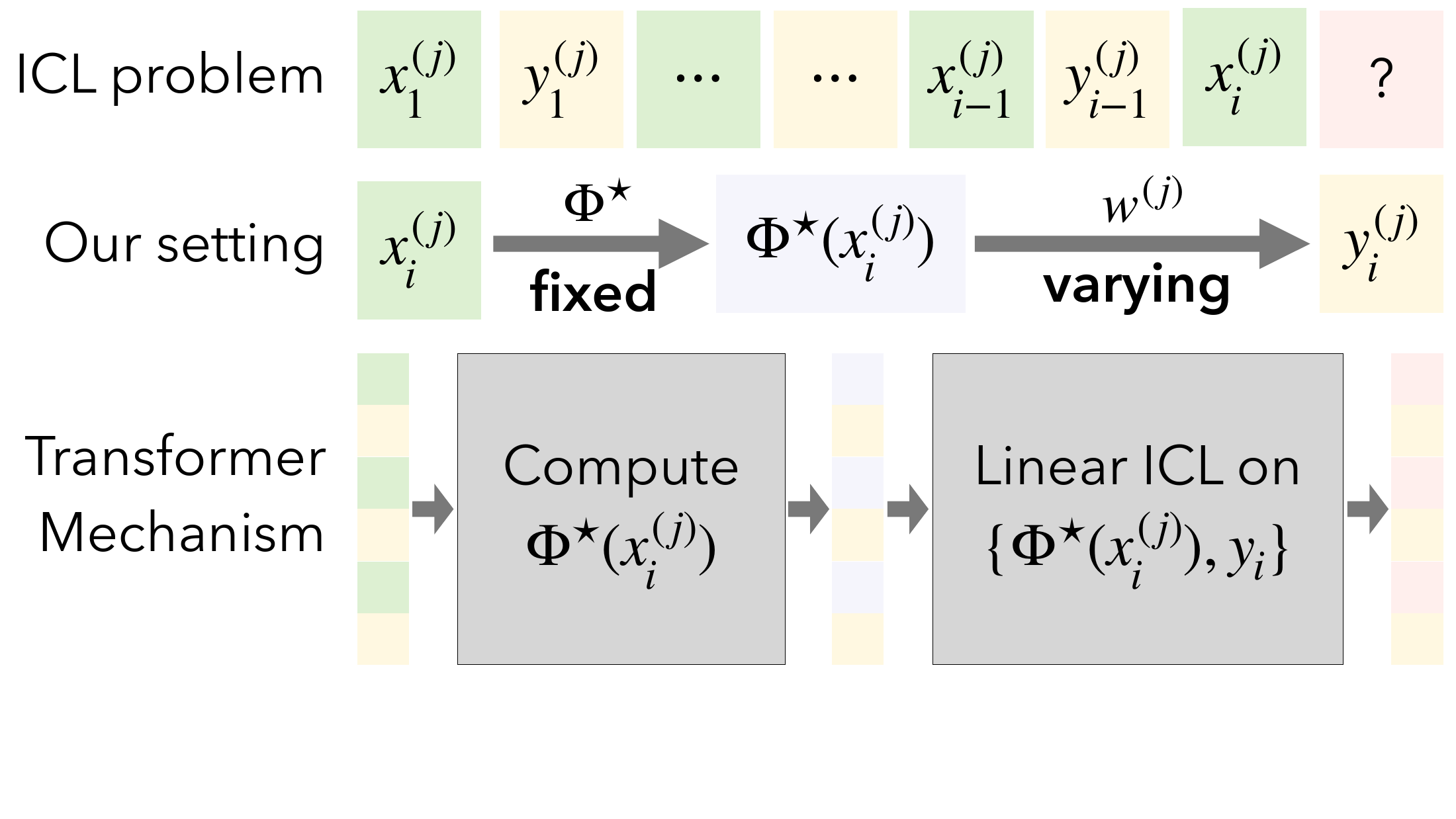}
  \end{minipage}
  \begin{minipage}{0.3\textwidth}
      \centering
      \subcaption{\small ICL risks}
      \label{fig:fig1-risk}
      \vspace{-.2em}
      \includegraphics[width=\linewidth]{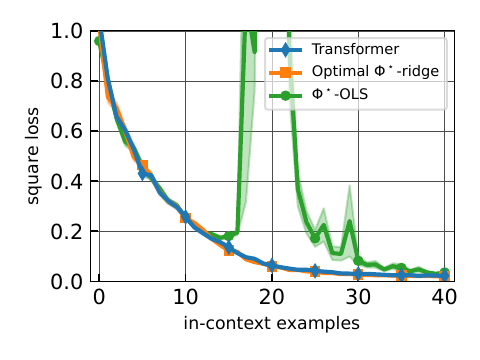}
  \end{minipage}
  \hspace{-1em}
  \begin{minipage}{0.3\textwidth}
      \centering
      \subcaption{\small Linear probes}
      \label{fig:fig1-probe}
      \vspace{-.1em}
      \includegraphics[width=\linewidth]{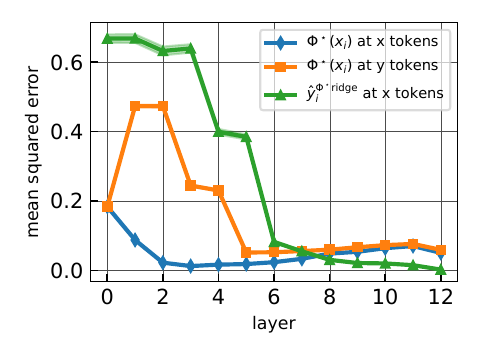}
  \end{minipage}
  \vspace{-1em}
  \caption{\small An illustration of our setting and results. {\bf (a)} We consider ICL problems with a fixed representation composed with changing linear functions, and we construct transformers that first compute the representations and then performs linear ICL. {\bf (b,c)} Empirically, learned transformers can perform near-optimal ICL in this setting, and exhibit mechanisms that align with our theory (detailed setups in~\cref{sec:exp-fixed}).
  }
  \label{figure:fig1}
  \vspace{-1em}
\end{figure}

\section{Related work}
\label{sec:related}

\paragraph{In-context learning} The in-context learning (ICL) capabilities of pretrained transformers have gained significant attention since first demonstrated with GPT-3 \citep{brown2020language}. Subsequent empirical studies have investigated the capabilities and limitations of ICL in large language models  \citep{liu2021makes, min2021noisy,  min2021metaicl, lu2021fantastically, zhao2021calibrate, rubin2021learning, razeghi2022impact, elhage2021mathematical, kirsch2022general,wei2023larger}. 

A line of recent work investigates why and how pretrained transformers perform ICL from a theoretical perspective \citep{garg2022can, li2023transformers, von2022transformers, akyurek2022learning, xie2021explanation, bai2023transformers, 
zhang2023trained, zhang2023and, ahn2023transformers, raventos2023pretraining}. In particular, \cite{xie2021explanation} proposed a Bayesian inference framework explaining ICL. \cite{garg2022can} showed transformers could be trained from scratch for ICL of simple function classes. Other studies found transformers can implement ICL through in-context gradient descent \citep{von2022transformers, akyurek2022learning} and in-context algorithm selection \citep{bai2023transformers}. \cite{zhang2023trained} studied the training dynamics of a single attention layer on linear ICL tasks. \cite{li2023dissecting} used the ICL framework to explain chain-of-thought reasoning \citep{wei2022chain}. Our work builds on and extends the work of~\citep{garg2022can,akyurek2022learning,von2022transformers,bai2023transformers}, where we study the more challenging setting of ICL with a representation function, and also provide new efficient ICL constructions for predicting at every token using a decoder transformer, as opposed to predicting only at the last token in most of these work.

\paragraph{In-weights learning versus in-context learning} 
Recent work has investigated when transformers learn a fixed input-label mapping versus when they perform ICL \citep{chan2022data, wei2023larger, bietti2023birth}.
\cite{chan2022data} refer to learning a fixed input-label mapping from the pre-training data as ``in-weights learning'' (IWL), in contrast with ICL. Our problem setting assumes the pre-training data admits a fixed representation function, which should be learned by IWL. In this perspective, unlike these existing works where IWL and ICL are typically treated as competing mechanisms, we study a model in which IWL (computing the fixed representation by transformer weights) and ICL (learning the changing linear function in context) occur simultaneously.

\paragraph{Mechanistic understanding and probing techniques} A line of work focuses on developing techniques for understanding the mechanisms of neural networks, in particular transformers~\citep{alain2016understanding, geiger2021causal, meng2022locating, von2022transformers, akyurek2022learning, wang2022interpretability, rauker2023toward}. We adopted the linear probing technique of~\citep{alain2016understanding} in a token-wise fashion for interpreting the ICL mechanisms of transformers. Beyond probing, more convincing mechanistic interpretations may require advanced approaches such as causal intervention~\citep{geiger2021causal, vig2020investigating, wang2022interpretability}; Our pasting experiment has a similar interventional flavor in that we feed input sequences (ICL instances) from another distribution directly (through a trainable embedding layer) to the upper module of a transformer.

\section{Preliminaries}
\label{sec:prelim}

\paragraph{Transformers}
We consider sequence-to-sequence functions applied to $N$ input vectors $\set{\bh_i}_{i=1}^N\subset \R^{\Dhid}$ in $\Dhid$ dimensions, which we write compactly as an input matrix $\bH=[\bh_1,\dots,\bh_N]\in \R^{\Dhid\times N}$, where each $\bh_i$ is a column of $\bH$ (also a \emph{token}). 

We use a standard $L$-layer decoder-only (autoregressive) transformer, which consists of $L$ consecutive blocks each with a masked self-attention layer (henceforth ``attention layer'') followed by an MLP layer. Each attention layer computes
\begin{talign*}
    \Attn_\btheta(\bH)\defeq \bH + \sum_{m=1}^M (\bV_m \bH) \times \barsig\paren{ \MSK \odot ((\bQ_m\bH)^\top (\bK_m\bH)) } \in \R^{D\times N},
\end{talign*}
where $\btheta=\sets{(\bQ_m, \bK_m, \bV_m)\subset\R^{\Dhid\times \Dhid}}_{m\in[M]}$ are the (query, key, value) matrices, $M$ is the number of heads, $\MSK\in\R^{N\times N}$ is the decoder mask matrix with $\MSK_{ij}=\indic{i\le j}$, and $\barsig$ is the activation function which is typically chosen as the (column-wise) softmax: $[\barsig(\bA)]_{:, j}=\softmax(\ba_j)\in\R^N$ for $\bA=[\ba_1, \dots, \ba_N]\in\R^{N\times N}$. Each MLP layer computes
\begin{talign*}
    \MLP_{\bW_1,\bW_2}(\bH) \defeq \bH + \bW_2\sigma(\bW_1\bH),
\end{talign*}
where $\bW_{\sets{1,2}}\in\R^{\Dhid\times \Dhid}$ are the weight matrices, and $\sigma(t)=\max\sets{t, 0}$ is the ReLU activation. We use $\TF$ to denote a transformer, and typically use $\tbH=\TF(\bH)$ to denote its output on $\bH$.

\paragraph{In-context learning}
We consider in-context learning (ICL) on regression problems, where each ICL instance is specified by a dataset $\cD=\sets{(\bx_i, y_i)}_{i\in[N]}\simiid \Pin$, with $(\bx_i, y_i)\in\R^d\times\R$, and the model is required to accurately predict $y_i$ given all past observations $\cD_{i-1}\defeq \sets{(\bx_j, y_j)}_{j\le i-1}$ and the test input $\bx_i$. The main difficulty of ICL compared with standard supervised learning is that each instance $\cD^{(j)}$ is in general drawn from a different data distribution $\Pin=\Pin^{(j)}$ (for example, a linear model with a new $\bw_\star^{(j)}\in\R^d$). Accurate prediction requires learning $\Pin$ in-context from the past observations $\cD_{i-1}$ (i.e. the context); merely memorizing any fixed $\Pin^{(j)}$ is not enough.

We consider using transformers to do ICL, where we feed a sequence of length $2N$ into the transformer $\TF$ using the following input format:
\begin{align}
\label{eqn:input-format}
\bH = [\bh_1, \dots, \bh_{2N}] = \begin{bmatrix}
\bx_1 & \bzero & \dots & \bx_N & \bzero \\
0 & y_1 & \dots & 0 & y_{N} \\
\bp^x_1 & \bp^y_1 & \dots & \bp^x_{N} & \bp^y_{N}
\end{bmatrix}
\in \R^{\Dhid\times 2N},
\end{align}
where $\bp_i^x,\bp_i^y\in\R^{\Dhid-d-1}$ are fixed positional encoding vectors consisting of \emph{zero paddings}, followed by non-zero entries containing information about the position index $i$ and indicator of being an $x$-token (1 in $\bp^x_i$, and $0$ in $\bp^y_i$); see~\cref{eqn:positional-encoding} for our concrete choice. We refer to each odd token $\bh_{2i-1}$ as as an $x$-token (also the $\bx_i$-token), and each even token $\bh_{2i}$ as a $y$-token (also the $y_i$-token). 

After obtaining the transformer output $\tbH = \TF(\bH)$, for every index $i\in[N]$, we extract the prediction $\hat{y}_i$ from the output token at position $\bx_i$:
$\hat{y}_i \defeq (\tbh^x_i)_{d+1}$.\footnote{There is no information leakage, as the ``prefix'' property of decoder transformers $\tbh^x_i=\tbh_{2i-1}=[\TF(\bH_{:, 1:(2i-1)})]_{2i-1}$ ensures that $\tbh^x_i$ (and thus $\hy_i$) only depends on $(\cD_{i-1}, \bx_i)$.} Feeding input~\cref{eqn:input-format} into the transformer simultaneously computes $\hy_i\leftarrow \TF(\bx_1,y_1,\dots,\bx_{i-1},y_{i-1},\bx_i)$ for all $i\in[N]$.

{\bf In addition} to the above setting, we also consider a \emph{dynamical system} setting with $\cD=\sets{\bx_i}_{i\in[N]}$ where the transformer predicts $\hbx_i$ from the preceding inputs $\bx_{\le i-1}$. See~\cref{sec:dynamical-system} for details.

\section{In-context learning with representations}
\label{sec:theory}

\subsection{Supervised learning with representation}
\label{sec:fixed-rep}

We begin by considering ICL on regression problems with representation, where labels depend on the input through linear functions of a fixed representation function. Formally, let $\Phi^\star:\R^d\to\R^D$ be a fixed representation function. We generate each in-context data distribution $\Pin=\Pin_\bw$ by sampling a linear function $\bw\sim\normal(\bzero, \tau^2\bI_D)$ from a Gaussian prior, and then generate the ICL instance $\cD=\sets{(\bx_i, y_i)}_{i\in[N]}\sim\Pin_\bw$ by a linear model on $\Phi^\star$ with coefficient $\bw$ and noise level $\sigma>0$:
\begin{align}
\label{eqn:model}
    y_i = \<\bw, \Phi^\star(\bx_i)\> + \sigma z_i,~~~\bx_i\simiid \Pin_x,~~~z_i\simiid \normal(0, 1),~~~i\in[N].
\end{align}
Note that all $\cD$'s share the same representation $\Phi^\star$, but each admits a unique linear function $\bw$.

The representation function $\Phi^\star$ can in principle be chosen arbitrarily. As a canonical and flexible choice for both our theory and experiments, we choose $\Phi^\star$ to be a standard $L$-layer MLP:
\begin{align}
\label{eqn:mlp}
    \Phi^\star(\bx) = \sigma^\star\paren{\bB_L^\star\sigma^\star\paren{\bB_{L-1}^\star\cdots\sigma^\star\paren{\bB_1^\star\bx}\cdots}}, \quad
    \bB_1^\star\in\R^{D\times d},~(\bB_\ell^\star)_{\ell=2}^L\subset\R^{D\times D}
\end{align}
where $D$ is the hidden and output dimension, and $\sigma^\star$ is the activation function (applied entry-wise) which we choose to be the leaky ReLU $\sigma^\star(t)=\sigma_\rho(t)\defeq \max\sets{t, \rho t}$ with slope $\rho\in(0,1)$.

\paragraph{Theory}
As $\Phi^\star$ is fixed and the $\bw$ is changing in model~\cref{eqn:model}, by construction, a good ICL algorithm should \emph{compute} the representations $\sets{\Phi^\star(\bx_i)}_i$ and perform linear ICL on the transformed dataset $\sets{(\Phi^\star(\bx_i), y_i)}_i$ to learn $\bw$. We consider the following class of \emph{$\Phi^\star$-ridge} estimators:
\begin{talign}
\label{eqn:phi-ridge}
\tag{$\Phi^\star$-Ridge}
    \hbwiphilam \defeq \argmin_{\bw\in\R^d} \frac{1}{2(i-1)}\sum_{j=1}^{i-1} \paren{ \<\bw, \Phi^\star(\bx_j)\> - y_j }^2 + \frac{\lambda}{2}\ltwo{\bw}^2,
\end{talign}
and we understand $\hat{\bw}_1^{\Phi^\star,\lambda}\defeq \bzero$. In words, $\hbwiphilam$ performs ridge regression on the transformed dataset $\sets{\Phi(\bx_j), y_j}_{j\le i-1}$ for all $i\in[N]$. By standard calculations, the Bayes-optimal predictor\footnote{The predictor $\hy_i=\hy_i(\cD_{i-1}, \bx_i)$ that minimizes the posterior square loss $\E[\frac{1}{2}(\hy_i - y_i)^2|\cD_{i-1},\bx_i]$.} for $y_i$ given $(\cD_{i-1}, \bx_i)$ is exactly the ridge predictor $\hyiphilam\defeq \langle \hbwiphilam, \Phi^\star(\bx_i)\rangle$ at $\lambda=\sigma^2/\tau^2$.

We show that there exists a transformer that can approximately implement~\cref{eqn:phi-ridge} in-context at every token $i\in[N]$. The proof can be found in~\cref{app:proof-fixed-rep}.
\begin{theorem}[Transformer can implement $\Phi^\star$-Ridge]
\label{thm:fixed-rep}
For any representation function $\Phi^\star$ of form~\cref{eqn:mlp}, any $\lambda>0$, $\Bphi,\Bw,\By>0$, $\eps<\Bphi\Bw/2$, letting $\kappa\defeq 1+\Bphi^2/\lambda$, there exists a transformer $\TF$ with $L+\cO\paren{\kappa\log(\Bphi\Bw/\eps)}$ layers, $5$ heads, $\Dhid=2D+d+10$ such that the following holds.

For any dataset $\cD$ such that $\ltwos{\Phi^\star(\bx_i)}\le \Bphi$, $|y_i|\le \By$ and the corresponding input $\bH\in\R^{\Dhid\times 2N}$ of format~\cref{eqn:input-format}, we have 
\begin{enumerate}[label=(\alph*),leftmargin=2em]
\item The first $(L+2)$ layers of $\TF$ transforms $\bx_i$ to the representation $\Phi^\star(\bx_i)$ at each $x$ token, and copies them into the succeeding $y$ token:
\begin{align}
\label{eqn:transformed-input}
    \TF^{(1:L+2)}(\bH) = \begin{bmatrix}
    \Phi^\star(\bx_1) & \Phi^\star(\bx_1) & \dots & \Phi^\star(\bx_N) & \Phi^\star(\bx_N) \\
    0 & y_1 & \dots & 0 & y_N \\
    \tbp^x_1 & \tbp^y_1 & \dots & \tbp^x_{N} & \tbp^y_{N}
    \end{bmatrix},
\end{align}
where $\tbp^x_i,\tbp^y_i$ only differ from $\bp^x_i,\bp^y_i$ in the dimension of the zero paddings.
\item For every index $i\in[N]$, the transformer output $\tbH=\TF(\bH)$ contains prediction $\hat{y}_i\defeq [\tbh_{2i-1}]_{D+1}$ that is close to the~\cref{eqn:phi-ridge} predictor: $\abss{\hat{y}_{i} - \langle \Phi^\star(\bx_i), \hbwiphilam \rangle} \le \eps$.
\end{enumerate}
\end{theorem}

The transformer construction in~\cref{thm:fixed-rep} consists of two ``modules'': The lower layers computes the representations and prepares the transformed dataset $\sets{(\Phi^\star(\bx_i), y_i)}_i$ into form~\cref{eqn:transformed-input}. In particular, each $\Phi^\star(\bx_i)$ appears both in the $i$-th $x$-token and is also copied into the succeeding $y$ token. The upper layers perform linear ICL (ridge regression) on top of the transformed dataset. We will test whether such mechanisms align with trained transformers in reality in our experiments (\cref{sec:exp-fixed}).

\paragraph{Proof techniques}
The proof of~\cref{thm:fixed-rep} builds upon (1) implementing the MLP $\Phi^\star$ by transformers (\cref{lem:mlp}), and (2) an efficient construction of in-context ridge regression (\cref{thm:ridge}), which to our knowledge is the first efficient construction for predicting \emph{at every token} using decoder transformers. The latter requires several new construction techniques such as a copying layer (\cref{lem:copy}), and an efficient implementation of $N$ parallel in-context gradient descent algorithms at all tokens simultaneously using a decoder transformer (\cref{prop:one-step-gd}). These extend the related constructions of~\cite{von2022transformers,bai2023transformers} who only consider predicting at the last token using encoder transformer, and could be of independent interest.

In addition, the bounds on the number of layers, heads, and $\Dhid$ in~\cref{thm:fixed-rep} can imply a sample complexity guarantee for (pre-)training: A transformer with $\wt{\eps}$-excess risk (on the same ICL instance distribution) over the one constructed in~\cref{thm:fixed-rep} can be found in $\tO\paren{(L+\kappa)^2(D+d)^2\wt{\eps}^{-2}}$ training instances, by the generalization analysis of~\cite[Theorem 20]{bai2023transformers}. We remark that the constructions in~\cref{thm:fixed-rep} \&~\ref{thm:dynamical-system} choose $\barsig$ as the normalized ReLU instead of softmax, following~\citep{bai2023transformers} and in resonance with recent empirical studies~\citep{wortsman2023replacing}.

\subsection{Dynamical system with representation}
\label{sec:dynamical-system}

As a variant of model~\cref{eqn:model}, we additionally consider a (nonlinear) dynamical system setting with data $\cD=(\bx_1,\dots,\bx_N)$, where each $\bx_{i+1}$ depends on the $k$ preceding inputs $[\bx_{i-k+1}; \dots; \bx_i]$ for some $k\ge 1$ through a linear function on top of a fixed representation function $\Phi^\star$. Compared to the supervised learning setting in~\cref{sec:fixed-rep}, this setting better resembles some aspects of natural language, where the next token in general depends on several preceding tokens.

Formally, let $k\ge 1$ denote the number of input tokens that the next token depends on, and $\Phi^\star:\R^{kd}\to\R^D$ denotes a representation function. Each ICL instance $\cD=\sets{\bx_i}_{i\in[N]}$ is generated as follows: First sample $\Pin=\Pin_\bW$ where $\bW\in\R^{D\times d}$ is sampled from a Gaussian prior: $W_{ij}\simiid \normal(0, \tau^2)$. Then sample the initial input $\bx_1\sim\Pin_x$ and let
\begin{equation}
\label{eqn:dynamical-system}
    \bx_{i+1} = \bW^\top \Phi^\star([\bx_{i-k+1}; \dots; \bx_{i}]) + \sigma \bz_i,~~~\bz_i\simiid \normal(\bzero, \bI_d),~~~i\in[N-1],
\end{equation}
where we understand $\bx_{j}\defeq \bzero_d$ for $j\le 0$. We choose $\Phi^\star$ to be the same $L$-layer MLP as in~\cref{eqn:mlp}, except that the first weight matrix has size $\bB_1^\star\in\R^{D\times kd}$ to be consistent with the dimension of the augmented input $\barbx_i\defeq [\bx_{i-k+1}; \dots; \bx_i]$. We remark that~\cref{eqn:dynamical-system} substantially generalizes the setting of~\cite{li2023transformers} which only considers~\emph{linear} dynamical systems (equivalent to $\Phi^\star\equiv {\rm id}$), a task arguably much easier for transformers to learn in context.

As $\bx_i$ acts as both inputs and labels in model~\cref{eqn:dynamical-system}, we use the following input format for transformers:
\begin{align}
\label{eqn:input-format-dynamical-system}
    \bH \defeq \begin{bmatrix}
        \bx_1 & \dots & \bx_N \\
        \bp_1 & \dots & \bp_N
    \end{bmatrix} \in \R^{\Dhid\times N},
\end{align}
where $\bp_i\defeq [\bzero_{\Dhid-d-4}; 1; i; i^2; i^3]$, and we extract prediction $\hat{\bx}_{i+1}$ from the $i$-th output token.

\paragraph{Theory}
Similar as above, we consider the ridge predictor for the dynamical system setting
\begin{talign}
\label{eqn:phi-ridge-dynamical-system}
\tag{$\Phi^\star$-Ridge-Dyn}
    \hbWiphilam \defeq \argmin_{\bW\in\R^{D\times d}} \frac{1}{2(i-1)}\sum_{j=1}^{i-1} \ltwo{ \bW^\top \Phi^\star(\barbx_j) - \bx_{j+1} }^2 + \frac{\lambda}{2}\lfro{\bW}^2.
\end{talign}
We understand $\hat{\bW}_0^{\Phi^\star,\lambda}\defeq \bzero_{D\times d}$, and let $\ltwoinfs{\bW}\defeq \max_{j\in[d]}\ltwos{\bW_{:, j}}$ for any $\bW\in\R^{D\times d}$. Again,~\cref{eqn:phi-ridge-dynamical-system} gives the Bayes-optimal predictor $(\hbWiphilam)^\top\Phi^\star(\barbx_i)$ at $\lambda=\sigma^2/\tau^2$.

The following result shows that~\cref{eqn:phi-ridge-dynamical-system} can also be implemented efficiently by a transformer. The proof can be found in~\cref{app:proof-dynamical-system}.
\begin{theorem}[Transformer can implement $\Phi^\star$-Ridge for dynamical system]
\label{thm:dynamical-system}

For the dynamical system setting where the $L$-layer representation function $\Phi^\star:\R^{kd}\to\R^D$ takes form~\cref{eqn:mlp}, but otherwise same settings as~\cref{thm:fixed-rep}, there exists a transformer $\TF$ with $L+2+\cO\paren{\kappa\log(\Bphi\Bw/\eps)}$ layers, $\max\sets{3d,5}$ heads, and $\Dhid=\max\sets{2(k+1),D}d+3(D+d)+5$ such that the following holds.

For any dataset $\cD$ such that $\ltwos{\Phi^\star(\barbx_i)}\le \Bphi$, $\linfs{\bx_i}\le \By$, and $\ltwoinfs{\hbWiphilam}\le \Bw/2$ (cf.~\cref{eqn:phi-ridge-dynamical-system}) for all $i\in[N]$, and corresponding input $\bH\in\R^{\Dhid\times N}$ of format~\cref{eqn:input-format-dynamical-system}, we have
\begin{enumerate}[label=(\alph*),leftmargin=2em]
\item The first transformer layer copies the $k$ previous inputs into the current token, and computes the first layer $\sets{\sigma_\rho(\bB_1^\star\barbx_i)}_{i\in[N]}$ within $\Phi^\star$:
\begin{talign}
    & \Attn^{(1)}(\bH) = \begin{bmatrix}
        \barbx_1 & \dots & \barbx_N \\
        \barbp_{1} & \dots & \barbp_N
    \end{bmatrix} = \begin{bmatrix}
    \bx_{1-k+1} & \dots & \bx_{N-k+1} \\
    \vert &  & \vert \\
    \bx_{1} & \dots & \bx_{N} \\
    \barbp_{1} & \dots & \barbp_{N}
    \end{bmatrix}; \label{eqn:copied-input-dynamical-system} \\
    & \TF^{(1)}(\bH) = \MLP^{(1)}\paren{\Attn^{(1)}(\bH)} = \begin{bmatrix}
        \sigma_\rho(\bB_1^\star\barbx_1) & \dots & \sigma_\rho(\bB_1^\star\barbx_N) \\
        \bx_1 & \dots & \bx_N \\
        \barbp_1' & \dots & \barbp_N'
    \end{bmatrix}. \label{eqn:relued-input-dynamical-system}
\end{talign}

\item The first $(L+1)$ layers of $\TF$ transforms each $\bx_i$ to $\Phi^\star(\barbx_i)$, and copies the preceding representation $\Phi^\star(\barbx_{i-1})$ onto the same token to form the (input, label) pair $(\Phi^\star(\barbx_{i-1}), \bx_i)$:
\begin{align}
\label{eqn:transformed-input-dynamical-system}
    \TF^{(1:L+1)}(\bH) = \begin{bmatrix}
    \Phi^\star(\barbx_1) & \Phi^\star(\barbx_2) & \dots  & \Phi^\star(\barbx_N) \\
    \bzero_d & \bzero_d & \dots & \bzero_d \\
    \bzero_D & \Phi^\star(\barbx_1) & \dots & \Phi^\star(\barbx_{N-1}) \\
    \bx_1 & \bx_2 & \dots & \bx_N \\
    \tbp_1 & \tbp_2 & \dots  & \tbp_N
    \end{bmatrix}.
\end{align}
Above, $\barbp_i,\barbp_i',\tbp_i$ only differs from $\bp_i$ in the dimension of the zero paddings.

\item For every index $i\in[N]$, the transformer output $\tbH=\TF(\bH)$ contains prediction $\hbx_{i+1}\defeq [\tbh_{i}]_{1:d}$ that is close to the~\cref{eqn:phi-ridge-dynamical-system} predictor: $\linfs{\hbx_{i+1} - (\hbWiphilam)^\top\Phi^\star(\barbx_i)} \le \eps$.
\end{enumerate}
\end{theorem}
To our best knowledge,~\cref{thm:dynamical-system} provides the first transformer construction for learning nonlinear dynamical systems in context. Similar as for~\cref{thm:fixed-rep}, the bounds on the transformer size here imply guarantees $\wt{\eps}$ excess risk within $\tO\paren{(L+\kappa)^2((k+D)d)^2\wt{\eps}^{-2}}$ (pre-)training instances.

In terms of the mechanisms, compared with~\cref{thm:fixed-rep}, the main differences in~\cref{thm:dynamical-system} are (1) the additional copying step~\cref{eqn:copied-input-dynamical-system} within the first layer, where the previous $(k-1)$ tokens $[\bx_{i-k+1}; \dots, \bx_{i-1}]$ are copied onto the $\bx_i$ token, to prepare for computing of $\Phi^\star(\barbx_i)$; (2) the intermediate output~\cref{eqn:transformed-input-dynamical-system}, where relevant information (for preparing for linear ICL) has form $[\Phi^\star(\barbx_{i-1}); \bx_i; \Phi^\star(\barbx_i)]$ and is gathered in the $\bx$-tokens, different from~\cref{eqn:transformed-input} where the relevant information is $[\Phi^\star(\bx_i); y_i]$, gathered in the $y$-token. We will test these in our experiments (\cref{sec:exp-dynamical-system}).

\section{Experiments}
\label{sec:exp}

We now empirically investigate trained transformers under the two settings considered in~\cref{sec:fixed-rep} \&~\ref{sec:dynamical-system}. In both cases, we choose the representation function $\Phi^\star$ to be a normalized version of the $L$-layer MLP~\cref{eqn:mlp}: $\Phi^\star(\bx)\defeq\tPhi^\star(\bx)/\ltwos{\tPhi^\star(\bx)}$, where $\tPhi^\star$ takes form~\cref{eqn:mlp}, with weight matrices $(\bB_i^ \star)_{i\in[L]}$ sampled as random (column/row)-orthogonal matrices and held fixed in each experiment, and slope $\rho=0.01$. We test $L\in\sets{1,2,3,4}$, hidden dimension $D\in\sets{5, 20, 80}$, and noise level $\sigma\in\sets{0, 0.1, 0.5}$. All experiments use $\Pin_x=\normal(\bzero,\bI_d)$, $\tau^2=1$, $d=20$, and $N=41$.

We use a small architecture within the GPT-2 family with 12 layers, 8 heads, and $\Dhid=256$, following~\citep{garg2022can,li2023transformers,bai2023transformers}. The (pre)-training objective for the transformer (for the supervised learning setting) is the average prediction risk at all tokens:
\begin{talign}\label{eqn:train-objective}
    \min_{\btheta} \E_{\bw,\cD\sim\Pin_\bw} \brac{\frac{1}{2N}\sum_{i=1}^N \paren{\hy_{\btheta, i}(\cD_{i-1}, \bx_i) - y_i}^2 },
\end{talign}
where $\hy_{\btheta, i}$ is extracted from the $(2i-1)$-th output token of $\TF_\btheta(\bH)$ (cf. Section~\ref{sec:prelim}). The objective for the dynamical system setting is defined similarly. Additional experimental details can be found in~\cref{app:exp-details}, and ablation studies (e.g. along the training trajectory; cf.~\cref{figure:train}) in~\cref{app:ablations}.

\begin{figure}[t]
  \centering
  \begin{minipage}{0.34\textwidth}
      \centering
      \subcaption{\small Varying noise level}
      \label{fig:fixed-rep-risk-noise}
      \vspace{-.2em}
      \includegraphics[width=\linewidth]{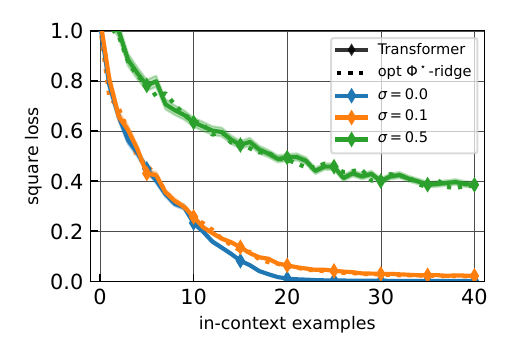}
  \end{minipage}
  \hspace{-1em}
  \begin{minipage}{0.32\textwidth}
      \centering
      \subcaption{\small Varying rep hidden dimension}
      \label{fig:fixed-rep-risk-D}
      \vspace{-.2em}
      \includegraphics[width=\linewidth]{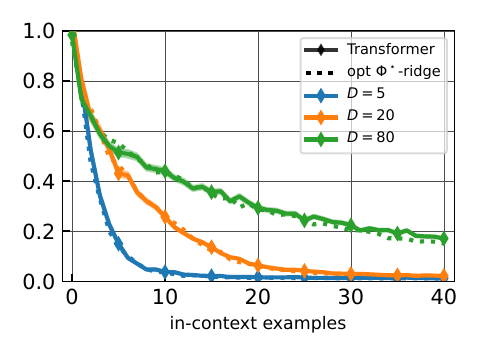}
  \end{minipage}
  \hspace{-1em}
  \begin{minipage}{0.32\textwidth}
      \centering
      \subcaption{\small Varying depth of rep}
      \label{fig:fixed-rep-risk-L}
      \vspace{-.2em}
      \includegraphics[width=\linewidth]{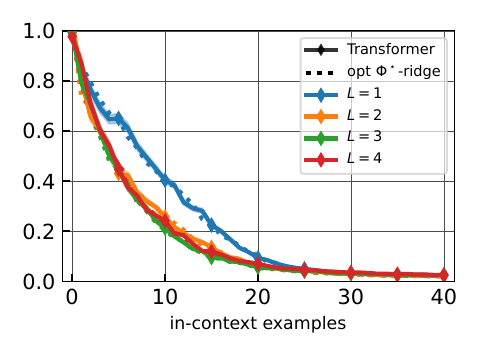}
  \end{minipage}
  \vspace{-1em}
  \caption{\small Test ICL risk for learning with representations. Each plot modifies a single problem parameter from the base setting $(L,D,\sigma)=(2, 20, 0.1)$. Dotted lines plot the Bayes-optimal risks for each setting respectively.
  }
  \label{figure:fixed-rep-risk}
  \vspace{-1em}
\end{figure}
\begin{figure}[t]
  \centering
  \begin{minipage}{0.34\textwidth}
      \centering
      \subcaption{
      \small Probe $\Phi^\star(\bx_i)$ at $\bx_i$ tokens %
      }
      \label{fig:fixed-rep-probe-z-x}
      \vspace{-.1em}
      \includegraphics[width=\linewidth]{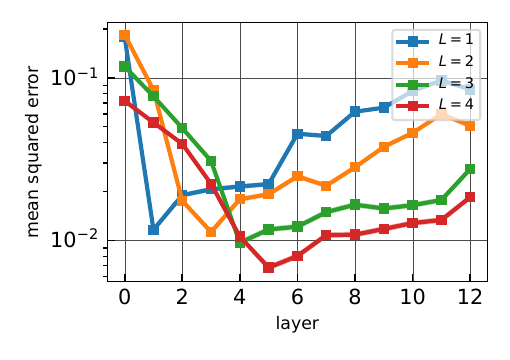}
  \end{minipage}
  \hspace{-1em}
  \begin{minipage}{0.32\textwidth}
      \centering
      \subcaption{\small Probe $\Phi^\star(\bx_i)$ at $y_i$ tokens}
      \label{fig:fixed-rep-probe-z-y}
      \vspace{-.1em}
      \includegraphics[width=\linewidth]{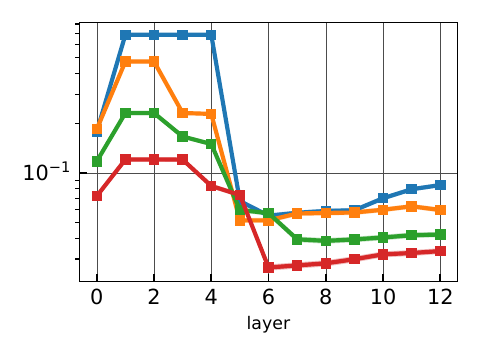}
  \end{minipage}
  \hspace{-1em}
  \begin{minipage}{0.33\textwidth}
      \centering
      \subcaption{\small Probe $\hy^{\Phi^\star,\lambda}_i$ at $\bx_i$ tokens}
      \label{fig:fixed-rep-probe-yhat-x}
      \vspace{-.1em}
      \includegraphics[width=\linewidth]{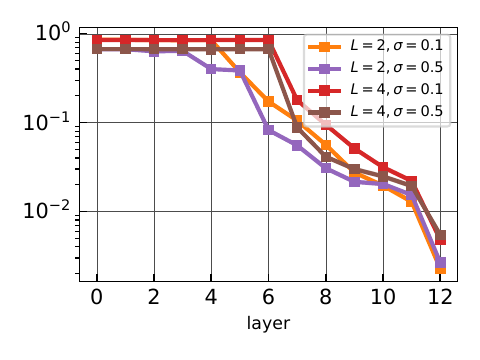}
  \end{minipage}
  \vspace{-1em}
  \caption{\small Probing errors for the learning with representation setting. Each setting modifies one or two problem parameters from the base setting $(L,D,\sigma)=(2, 20, 0.1)$. Note that the orange  curve corresponds to the same setting (and thus the same transformer) across (a,b,c), as well as the red curve.}
  \label{figure:fixed-rep-probe}
  \vspace{-1em}
\end{figure}

\subsection{Supervised learning with representation}
\label{sec:exp-fixed}

We first test ICL with supervised learning data as in~\cref{sec:fixed-rep}, where for each configuration of $(L,D,\sigma)$ (which induces a $\Phi^\star$) we train a transformer on ICL data distribution~\cref{eqn:model} and evaluate ICL on the same distribution. Note that \cref{fig:fig1-probe} \&~\ref{fig:fig1-risk} plots the results for $(L,D,\sigma)=(2,20,0.1)$.

\paragraph{ICL performance} 
\cref{figure:fixed-rep-risk} reports the test risk across various settings, where we observe that trained transformers can consistently match the Bayes-optimal ridge predictor. This extends existing results which show that linear functions (without a representation) can be learned near-optimally in-context by transformers~\citep{garg2022can,akyurek2022learning}, adding our model~\cref{eqn:model} to this list of (empirically) nearly-optimally learnable function classes. Among the complexity measures $(L,D,\sigma)$, observe that the noise level $\sigma$ and hidden dimension $D$ of the representation (\cref{fig:fixed-rep-risk-noise} \&~\ref{fig:fixed-rep-risk-D}) appears to have a larger effect on the (nearly Bayes-optimal) risk than the depth $L$ (Figure~\ref{fig:fixed-rep-risk-L}).

\paragraph{Mechanisms via linear probing}
We conduct probing experiments to further understand the mechanisms of the trained transformers. In accordance with the theoretical construction in~\cref{thm:fixed-rep}, our main question here is: Does the trained transformer perform the following in order:
\begin{enumerate}[leftmargin=2em, topsep=0pt, itemsep=0pt]
    \item Computes $\Phi^\star(\bx_i)$ at $x_i$ tokens;
    \item Copies them onto the following $y_i$ token and obtains dataset $\sets{\Phi^\star(\bx_i), y_i}_i$ in the form of~\cref{eqn:transformed-input};
    \item Performs linear ICL on top of $\sets{\Phi^\star(\bx_i), y_i}_i$?
\end{enumerate}
While such internal mechanisms are in general difficult to quantify exactly, we adapt the \emph{linear probing}~\citep{alain2016understanding} technique to the transformer setting to identify evidence. Linear probing allows us to test whether intermediate layer outputs (tokens) $\sets{\bh^{x,(\ell)}_i}_{\ell\in[12]}$ ($\ell$ denotes the layer) and $\sets{\bh^{y,(\ell)}_i}_{\ell\in[12]}$ ``contains'' various quantities of interest, by linearly regressing these quantities (as the y) on the intermediate tokens (as the x), pooled over the token index $i\in[N]$. For example, regressing $\Phi^\star(\bx_i)$ on $\bh^{x,(\ell)}_i$ tests whether the $\bx_i$ token after the $\ell$-th layer ``contains'' $\Phi^\star(\bx_i)$, where a smaller error indicates a better containment. See~\cref{app:probe-details} for further setups of linear probing.

\cref{figure:fixed-rep-probe} reports the errors of three linear probes across all 12 layers: The representation $\Phi^\star(\bx_i)$ in the $\bx_i$ tokens and $y_i$ tokens, and the optimal ridge prediction $\hy^{\Phi^\star,\lambda}_i$ in the $\bx_i$ tokens. Observe that the probing errors for the representation decrease through lower layers and then increase through upper layers (\cref{fig:fixed-rep-probe-z-x} \&~\ref{fig:fixed-rep-probe-z-y}), whereas probing errors for the ridge prediction monotonically decrease through the layers (\cref{fig:fixed-rep-probe-yhat-x}), aligning with our construction that the transformer first computes the representations and then performs ICL on top of the representation. Also note that deeper representations take more layers to compute (\cref{fig:fixed-rep-probe-z-x}). Further, the representation shows up later in the $y$-tokens (layers 5-6) than in the $x$-tokens (layers 1,3,4,5), consistent with the copying mechanism, albeit the copying appears to be lossy (probe errors are higher at $y$-tokens). 

Finally, observe that the separation between the lower and upper modules seems to be strong in certain runs---For example, the red transformer ($L=4,\sigma=0.1$) computes the representation at layer $5$, copies them onto $y$-tokens at layer $6$, and starts to perform iterative ICL from layer $7$, which aligns fairly well with our theoretical constructions at a high level.

\begin{figure}[t]
  \centering
  \begin{minipage}{0.45\textwidth}
      \centering
      \subcaption{
      \small Illustration of the pasting experiment
      }
      \label{fig:paste-illustration}
      \vspace{.4em}
      \includegraphics[width=\linewidth]{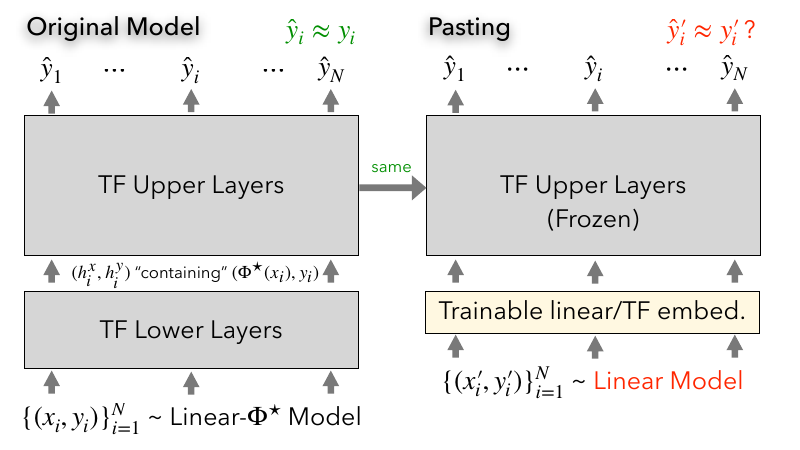}
  \end{minipage}
  \hspace{2em}
  \begin{minipage}{0.35\textwidth}
      \centering
      \subcaption{\small Linear ICL in TF\_upper via pasting}
      \label{fig:paste-risk}
      \vspace{0em}
      \includegraphics[width=1\linewidth]{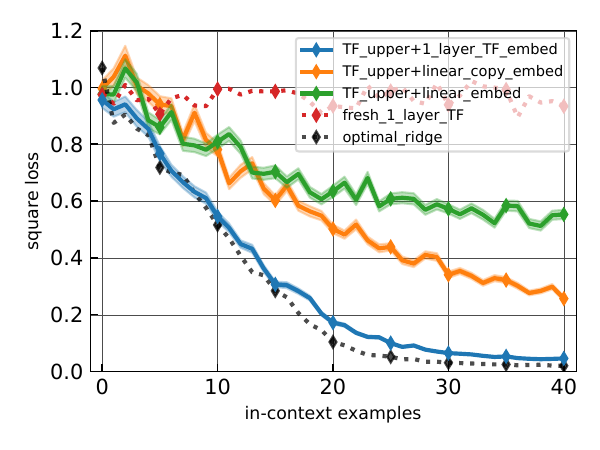}
  \end{minipage}
  \vspace{-1em}
  \caption{\small {\bf (a)} Illustration of our pasting experiment, which examines the linear ICL capability of the upper module of a trained transformer. {\bf (b)} Pasting results for the upper module of a trained transformer in setting $(L, D, \sigma)=(3, 20, 0.1)$. It achieves nearly optimal linear ICL risk (in $20$ dimension with noise $0.1$), using a 1-layer transformer embedding, and also non-trivial performance using linear and linear copy embeddings.
  }
  \label{figure:paste}
  \vspace{-1em}
\end{figure}

\paragraph{Investigating upper module via pasting}
To further investigate upper module, we test whether it is indeed a strong ICL learner \emph{on its own} without relying on the lower module, which would provide stronger evidence that the upper module performs linear ICL. However, a key challenge here is that it is unclear how to feed raw inputs directly into the upper module, as they supposedly only admit input formats emitted from the lower module---the part we wanted to exclude in the first place.

We address this by conducting a \emph{pasting} experiment, where we feed $D$-dimensional \emph{linear ICL problems} ($y_i'=\langle \bw', \bx_i' \rangle$ \emph{without} a representation) with input format~\cref{eqn:input-format} directly to the upper module of the transformer trained on representation $\Phi^\star$, by adding a \emph{trainable embedding layer} in between; see~\cref{fig:paste-illustration} for an illustration of the pasting approach. This trainable embedding layer itself needs to be shallow without much ICL power---we test the following three choices: (1) \emph{Linear} embedding: $\barbh^x_i=\bW[\bx_i; 0]$ and $\bh^y_i=\bW[\bzero_D; y_i]$; (2) \emph{Linear-copy} embedding, where the $y$ tokens are instead $\barbh^y_i=\bW[\bx_i; y_i]$, motivated by the format~\cref{eqn:transformed-input}; (3) \emph{One-layer transformer} embedding $\barTF$, which computes $\barbH=\barTF(\bH)$. See~\cref{app:pasting-details} for further setups of pasting.

\cref{fig:paste-risk} shows the pasting results on a trained transformer on $(L,D,\sigma)=(3,20,0.1)$ (an ablation in~\cref{fig:paste-ablations}), where we dissect the lower and upper modules at layer 4 as suggested by the probing curve (\cref{fig:fixed-rep-probe-z-x} green). Perhaps surprisingly, the upper module of the transformer can indeed perform nearly optimal linear ICL without representation when we use the one-layer transformer embedding. Note that a (freshly trained) single-layer transformer itself performs badly, achieving about the trivial test risk $1.01$, which is expected due to our specific input format\footnote{A one-layer transformer does not have much ICL power using input format~\cref{eqn:input-format}---$\bx_i$ and $y_i$ are stored in separate tokens there, which makes ``one-layer'' mechanisms such as gradient descent~\citep{von2022transformers,akyurek2022learning,bai2023transformers} unlikely to be implementable; see~\cref{app:inability} for a discussion.}~\cref{eqn:input-format}. This suggests that the majority of the ICL is indeed carried by the upper module, with the one-layer transformer embedding not doing much ICL itself. Also note that the linear-copy and linear embeddings also yield reasonable (though suboptimal) performance, with linear-copy performing slightly better.

\subsubsection{Extension: Mixture of multiple representations}
\label{sec:mtl-short}

We aditionally investigate an harder scenario in which there exists \emph{multiple possible representation functions} $(\Phi^\star_j)_{j\in[K]}$, and the ICL data distribution is a mixture of the $K$ distributions of form~\cref{eqn:model} each induced by $\Phi^\star_j$ (equivalent to using the concatenated representation $\barPhi^\star=[\Phi^\star_1,\dots,\Phi^\star_K]$ with a group $1$-sparse prior on $\barbw\in\R^{KD}$). We find that transformers still approach Bayes-optimal risks, though less so compared with the single-representation setting. Using linear probes, we find that transformers sometimes implement the \emph{post-ICL algorithm selection} mechanism identified in~\cite{bai2023transformers}, depending on the setting. Details are deferred to~\cref{app:mtl} due to the space limit.

\begin{figure}[t]
  \centering
  \begin{minipage}{0.34\textwidth}
      \centering
      \subcaption{\small Risk}
      \label{fig:dynamics-risk}
      \vspace{-.25em}
      \includegraphics[width=\linewidth]{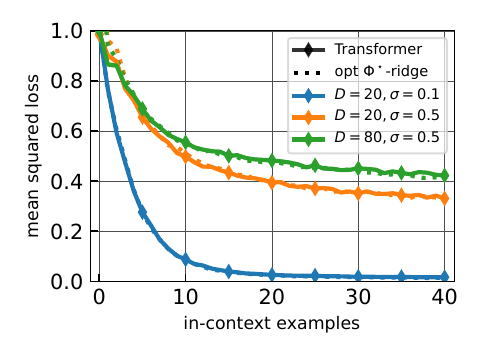}
  \end{minipage}
  \hspace{-1em}
  \begin{minipage}{0.34\textwidth}
      \centering
      \subcaption{\small Probe past inputs at $\bx_i$ tokens}
      \label{fig:dynamics-copy1}
      \vspace{-.1em}
      \includegraphics[width=\linewidth]{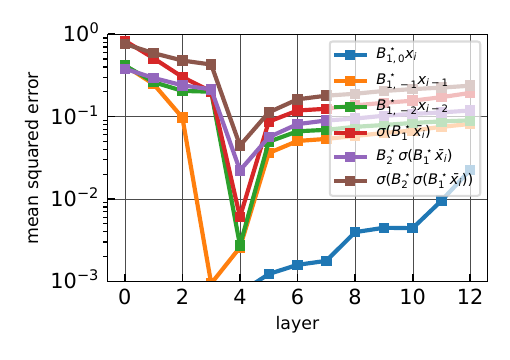}
  \end{minipage}
  \hspace{-1em}
  \begin{minipage}{0.32\textwidth}
      \centering
      \subcaption{\small Probe $\Phi^\star(\bx_{i-j})$ at $\bx_i$ tokens}
      \label{fig:dynamics-copy2}
      \vspace{.4em}
      \includegraphics[width=\linewidth]{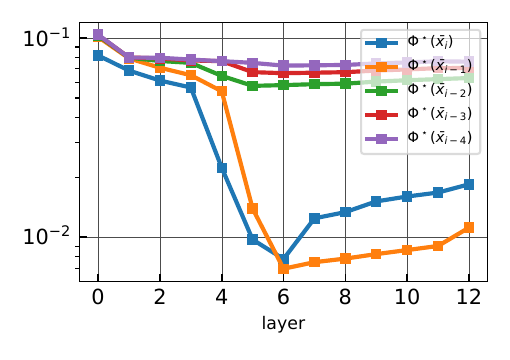}
  \end{minipage}
  \vspace{-1em}
  \caption{\small ICL risks and probing errors for the dynamical system setting. {\bf (a)} Each curve modifies problem parameters from the base setting $(k,L,D,\sigma)=(3,2,80, 0.5)$. {\bf (b,c)} Results are with the same base setting.
  }
  \label{figure:dynamics}
  \vspace{-1em}
\end{figure}

\subsection{Dynamical systems}
\label{sec:exp-dynamical-system}

We now study the dynamical systems setting in~\cref{sec:dynamical-system} using the same approaches as in~\cref{sec:exp-fixed}.~\cref{fig:dynamics-risk} shows that transformers can still consistently achieve nearly Bayes-optimal ICL risk. An ablation of the risks and probing errors in alternative settings can be found in~\cref{app:ablations-dynamical-system}.

\paragraph{Probing copying mechanisms}
The main mechanistic question we ask here is about the data preparation phase, where the transformer construction in~\cref{thm:dynamical-system} performs copying \emph{twice}:
\begin{enumerate}[label=\roman*),leftmargin=2em, topsep=0pt, itemsep=0pt]
    \item A copying of $[\bx_{i-k+1};\dots;\bx_{i-1}]$ onto the $\bx_i$ token as in~\cref{eqn:copied-input-dynamical-system}, to prepare for the computation of $\Phi^\star(\barbx_i)$; As copying may not be distinguishable from the consequent \emph{matrix multiplication} step $[\bx_{i-k+1};\dots,\bx_{i-1};\bx_i]\\\mapsto \bB_1^\star[\bx_{i-k+1};\dots,\bx_{i-1};\bx_i]$, we probe instead the result $\B_{1,-j}^\star\bx_{i-j}$ after matrix multiplication, where $\bB_{1,-j}^\star\in\R^{D\times d}$ denotes the block within $\bB_1^\star$ hitting $\bx_{i-j}$.
    \item A second copying of $\Phi^\star(\barbx_{i-1})$ onto the $\bx_i$ token to obtain~\cref{eqn:transformed-input-dynamical-system}, after $\sets{\Phi^\star(\barbx_i)}_i$ are computed.
\end{enumerate}

We probe one transformer trained on the dynamical systems problem with $k=3$ (so that the useful preceding inputs are $\bx_{i-1}$ and $\bx_{i-2}$), and find that the transformer indeed performs the two conjectured copyings.~\cref{fig:dynamics-copy1} demonstrates copying i) onto the current token, where the copying of $\bx_{i-1}$ happens earlier (at layer 3) and is slightly more accurate than that of $\bx_{i-2}$ (at layer 4), as expected. Further observe that layer 4 (which we recall contains an attention layer and an MLP layer) have seemingly also implemented the (unnormalized) MLP representation $\tPhi^\star(\barbx_i)=\sigma_\rho(\bB_2^\star\sigma_\rho(\bB_1^\star\barbx_i))$, though the probing error for the actual representation $\Phi^\star(\barbx_i)=\tPhi^\star(\barbx_i)/\ltwos{\tPhi^\star(\barbx_i)}$ continues to drop in layer 4-6 (\cref{fig:dynamics-copy2}). \cref{fig:dynamics-copy2} further demonstrates copying ii), where $\Phi^\star(\barbx_{i-1})$ are indeed copied to the $i$-th token, whereas by sharp contrast $\Phi^\star(\barbx_{i-k})$ for $k\ge 2$ are \emph{not} copied at all into the $\bx_i$ token, aligning with our conjectured intermediate output format~\cref{eqn:transformed-input-dynamical-system}.

\section{Conclusion}

This paper presents theoretical and mechanistic studies on the in-context learning ability of transformers on learning tasks involving representation functions, where we give efficient transformer constructions for linear ICL on top of representations for the supervised learning and dynamical system setting, and empirically confirm the existence of various high-level mechanisms in trained transformers. We believe our work opens up the investigation of ICL beyond simple function classes, and suggests open questions such as further investigations of the mechanisms of the linear ICL modules, and theory for ICL in more complex function classes. One limitation of our work is that the setting still consists of synthetic data with idealistic representation functions; performing similar studies on more real-world data would be an important direction for future work.

\bibliographystyle{plainnat}
\bibliography{bib}

\begin{thebibliography}{40}
\providecommand{\natexlab}[1]{#1}
\providecommand{\url}[1]{\texttt{#1}}
\expandafter\ifx\csname urlstyle\endcsname\relax
  \providecommand{\doi}[1]{doi: #1}\else
  \providecommand{\doi}{doi: \begingroup \urlstyle{rm}\Url}\fi

\bibitem[Ahn et~al.(2023)Ahn, Cheng, Daneshmand, and Sra]{ahn2023transformers}
Kwangjun Ahn, Xiang Cheng, Hadi Daneshmand, and Suvrit Sra.
\newblock Transformers learn to implement preconditioned gradient descent for in-context learning.
\newblock \emph{arXiv preprint arXiv:2306.00297}, 2023.

\bibitem[Aky{\"u}rek et~al.(2022)Aky{\"u}rek, Schuurmans, Andreas, Ma, and Zhou]{akyurek2022learning}
Ekin Aky{\"u}rek, Dale Schuurmans, Jacob Andreas, Tengyu Ma, and Denny Zhou.
\newblock What learning algorithm is in-context learning? investigations with linear models.
\newblock \emph{arXiv preprint arXiv:2211.15661}, 2022.

\bibitem[Alain and Bengio(2016)]{alain2016understanding}
Guillaume Alain and Yoshua Bengio.
\newblock Understanding intermediate layers using linear classifier probes.
\newblock \emph{arXiv preprint arXiv:1610.01644}, 2016.

\bibitem[Bai et~al.(2023)Bai, Chen, Wang, Xiong, and Mei]{bai2023transformers}
Yu~Bai, Fan Chen, Huan Wang, Caiming Xiong, and Song Mei.
\newblock Transformers as statisticians: Provable in-context learning with in-context algorithm selection.
\newblock \emph{arXiv preprint arXiv:2306.04637}, 2023.

\bibitem[Bietti et~al.(2023)Bietti, Cabannes, Bouchacourt, Jegou, and Bottou]{bietti2023birth}
Alberto Bietti, Vivien Cabannes, Diane Bouchacourt, Herve Jegou, and Leon Bottou.
\newblock Birth of a transformer: A memory viewpoint.
\newblock \emph{arXiv preprint arXiv:2306.00802}, 2023.

\bibitem[Brown et~al.(2020)Brown, Mann, Ryder, Subbiah, Kaplan, Dhariwal, Neelakantan, Shyam, Sastry, Askell, et~al.]{brown2020language}
Tom Brown, Benjamin Mann, Nick Ryder, Melanie Subbiah, Jared~D Kaplan, Prafulla Dhariwal, Arvind Neelakantan, Pranav Shyam, Girish Sastry, Amanda Askell, et~al.
\newblock Language models are few-shot learners.
\newblock \emph{Advances in neural information processing systems}, 33:\penalty0 1877--1901, 2020.

\bibitem[Bubeck(2015)]{bubeck2015convex}
S{\'e}bastien Bubeck.
\newblock Convex optimization: Algorithms and complexity.
\newblock \emph{Foundations and Trends{\textregistered} in Machine Learning}, 8\penalty0 (3-4):\penalty0 231--357, 2015.

\bibitem[Bubeck et~al.(2023)Bubeck, Chandrasekaran, Eldan, Gehrke, Horvitz, Kamar, Lee, Lee, Li, Lundberg, et~al.]{bubeck2023sparks}
S{\'e}bastien Bubeck, Varun Chandrasekaran, Ronen Eldan, Johannes Gehrke, Eric Horvitz, Ece Kamar, Peter Lee, Yin~Tat Lee, Yuanzhi Li, Scott Lundberg, et~al.
\newblock Sparks of artificial general intelligence: Early experiments with gpt-4.
\newblock \emph{arXiv preprint arXiv:2303.12712}, 2023.

\bibitem[Chan et~al.(2022)Chan, Santoro, Lampinen, Wang, Singh, Richemond, McClelland, and Hill]{chan2022data}
Stephanie Chan, Adam Santoro, Andrew Lampinen, Jane Wang, Aaditya Singh, Pierre Richemond, James McClelland, and Felix Hill.
\newblock Data distributional properties drive emergent in-context learning in transformers.
\newblock \emph{Advances in Neural Information Processing Systems}, 35:\penalty0 18878--18891, 2022.

\bibitem[Dai et~al.(2022)Dai, Sun, Dong, Hao, Sui, and Wei]{dai2022can}
Damai Dai, Yutao Sun, Li~Dong, Yaru Hao, Zhifang Sui, and Furu Wei.
\newblock Why can gpt learn in-context? language models secretly perform gradient descent as meta optimizers.
\newblock \emph{arXiv preprint arXiv:2212.10559}, 2022.

\bibitem[Elhage et~al.(2021)Elhage, Nanda, Olsson, Henighan, Joseph, Mann, Askell, Bai, Chen, Conerly, et~al.]{elhage2021mathematical}
N~Elhage, N~Nanda, C~Olsson, T~Henighan, N~Joseph, B~Mann, A~Askell, Y~Bai, A~Chen, T~Conerly, et~al.
\newblock A mathematical framework for transformer circuits.
\newblock \emph{Transformer Circuits Thread}, 2021.

\bibitem[Garg et~al.(2022)Garg, Tsipras, Liang, and Valiant]{garg2022can}
Shivam Garg, Dimitris Tsipras, Percy~S Liang, and Gregory Valiant.
\newblock What can transformers learn in-context? a case study of simple function classes.
\newblock \emph{Advances in Neural Information Processing Systems}, 35:\penalty0 30583--30598, 2022.

\bibitem[Geiger et~al.(2021)Geiger, Lu, Icard, and Potts]{geiger2021causal}
Atticus Geiger, Hanson Lu, Thomas Icard, and Christopher Potts.
\newblock Causal abstractions of neural networks.
\newblock \emph{Advances in Neural Information Processing Systems}, 34:\penalty0 9574--9586, 2021.

\bibitem[Kirsch et~al.(2022)Kirsch, Harrison, Sohl-Dickstein, and Metz]{kirsch2022general}
Louis Kirsch, James Harrison, Jascha Sohl-Dickstein, and Luke Metz.
\newblock General-purpose in-context learning by meta-learning transformers.
\newblock \emph{arXiv preprint arXiv:2212.04458}, 2022.

\bibitem[Li et~al.(2023{\natexlab{a}})Li, Ildiz, Papailiopoulos, and Oymak]{li2023transformers}
Yingcong Li, M~Emrullah Ildiz, Dimitris Papailiopoulos, and Samet Oymak.
\newblock Transformers as algorithms: Generalization and implicit model selection in in-context learning.
\newblock \emph{arXiv preprint arXiv:2301.07067}, 2023{\natexlab{a}}.

\bibitem[Li et~al.(2023{\natexlab{b}})Li, Sreenivasan, Giannou, Papailiopoulos, and Oymak]{li2023dissecting}
Yingcong Li, Kartik Sreenivasan, Angeliki Giannou, Dimitris Papailiopoulos, and Samet Oymak.
\newblock Dissecting chain-of-thought: A study on compositional in-context learning of mlps.
\newblock \emph{arXiv preprint arXiv:2305.18869}, 2023{\natexlab{b}}.

\bibitem[Liu et~al.(2021)Liu, Shen, Zhang, Dolan, Carin, and Chen]{liu2021makes}
Jiachang Liu, Dinghan Shen, Yizhe Zhang, Bill Dolan, Lawrence Carin, and Weizhu Chen.
\newblock What makes good in-context examples for gpt-$3 $?
\newblock \emph{arXiv preprint arXiv:2101.06804}, 2021.

\bibitem[Lu et~al.(2021)Lu, Bartolo, Moore, Riedel, and Stenetorp]{lu2021fantastically}
Yao Lu, Max Bartolo, Alastair Moore, Sebastian Riedel, and Pontus Stenetorp.
\newblock Fantastically ordered prompts and where to find them: Overcoming few-shot prompt order sensitivity.
\newblock \emph{arXiv preprint arXiv:2104.08786}, 2021.

\bibitem[Meng et~al.(2022)Meng, Bau, Andonian, and Belinkov]{meng2022locating}
Kevin Meng, David Bau, Alex Andonian, and Yonatan Belinkov.
\newblock Locating and editing factual associations in gpt.
\newblock \emph{Advances in Neural Information Processing Systems}, 35:\penalty0 17359--17372, 2022.

\bibitem[Min et~al.(2021{\natexlab{a}})Min, Lewis, Hajishirzi, and Zettlemoyer]{min2021noisy}
Sewon Min, Mike Lewis, Hannaneh Hajishirzi, and Luke Zettlemoyer.
\newblock Noisy channel language model prompting for few-shot text classification.
\newblock \emph{arXiv preprint arXiv:2108.04106}, 2021{\natexlab{a}}.

\bibitem[Min et~al.(2021{\natexlab{b}})Min, Lewis, Zettlemoyer, and Hajishirzi]{min2021metaicl}
Sewon Min, Mike Lewis, Luke Zettlemoyer, and Hannaneh Hajishirzi.
\newblock Metaicl: Learning to learn in context.
\newblock \emph{arXiv preprint arXiv:2110.15943}, 2021{\natexlab{b}}.

\bibitem[Min et~al.(2022)Min, Lyu, Holtzman, Artetxe, Lewis, Hajishirzi, and Zettlemoyer]{min2022rethinking}
Sewon Min, Xinxi Lyu, Ari Holtzman, Mikel Artetxe, Mike Lewis, Hannaneh Hajishirzi, and Luke Zettlemoyer.
\newblock Rethinking the role of demonstrations: What makes in-context learning work?
\newblock \emph{arXiv preprint arXiv:2202.12837}, 2022.

\bibitem[Olsson et~al.(2022)Olsson, Elhage, Nanda, Joseph, DasSarma, Henighan, Mann, Askell, Bai, Chen, et~al.]{olsson2022context}
Catherine Olsson, Nelson Elhage, Neel Nanda, Nicholas Joseph, Nova DasSarma, Tom Henighan, Ben Mann, Amanda Askell, Yuntao Bai, Anna Chen, et~al.
\newblock In-context learning and induction heads.
\newblock \emph{arXiv preprint arXiv:2209.11895}, 2022.

\bibitem[OpenAI(2023)]{openai2023gpt}
OpenAI.
\newblock Gpt-4 technical report.
\newblock \emph{arXiv preprint arXiv:2303.08774}, 2023.

\bibitem[R{\"a}uker et~al.(2023)R{\"a}uker, Ho, Casper, and Hadfield-Menell]{rauker2023toward}
Tilman R{\"a}uker, Anson Ho, Stephen Casper, and Dylan Hadfield-Menell.
\newblock Toward transparent ai: A survey on interpreting the inner structures of deep neural networks.
\newblock In \emph{2023 IEEE Conference on Secure and Trustworthy Machine Learning (SaTML)}, pages 464--483. IEEE, 2023.

\bibitem[Ravent{\'o}s et~al.(2023)Ravent{\'o}s, Paul, Chen, and Ganguli]{raventos2023pretraining}
Allan Ravent{\'o}s, Mansheej Paul, Feng Chen, and Surya Ganguli.
\newblock Pretraining task diversity and the emergence of non-bayesian in-context learning for regression.
\newblock \emph{arXiv preprint arXiv:2306.15063}, 2023.

\bibitem[Razeghi et~al.(2022)Razeghi, Logan~IV, Gardner, and Singh]{razeghi2022impact}
Yasaman Razeghi, Robert~L Logan~IV, Matt Gardner, and Sameer Singh.
\newblock Impact of pretraining term frequencies on few-shot reasoning.
\newblock \emph{arXiv preprint arXiv:2202.07206}, 2022.

\bibitem[Rubin et~al.(2021)Rubin, Herzig, and Berant]{rubin2021learning}
Ohad Rubin, Jonathan Herzig, and Jonathan Berant.
\newblock Learning to retrieve prompts for in-context learning.
\newblock \emph{arXiv preprint arXiv:2112.08633}, 2021.

\bibitem[Shen et~al.(2023)Shen, Guo, Tan, Tang, Wang, and Bian]{shen2023study}
Kai Shen, Junliang Guo, Xu~Tan, Siliang Tang, Rui Wang, and Jiang Bian.
\newblock A study on relu and softmax in transformer.
\newblock \emph{arXiv preprint arXiv:2302.06461}, 2023.

\bibitem[Touvron et~al.(2023)Touvron, Martin, Stone, Albert, Almahairi, Babaei, Bashlykov, Batra, Bhargava, Bhosale, et~al.]{touvron2023llama}
Hugo Touvron, Louis Martin, Kevin Stone, Peter Albert, Amjad Almahairi, Yasmine Babaei, Nikolay Bashlykov, Soumya Batra, Prajjwal Bhargava, Shruti Bhosale, et~al.
\newblock Llama 2: Open foundation and fine-tuned chat models.
\newblock \emph{arXiv preprint arXiv:2307.09288}, 2023.

\bibitem[Vig et~al.(2020)Vig, Gehrmann, Belinkov, Qian, Nevo, Singer, and Shieber]{vig2020investigating}
Jesse Vig, Sebastian Gehrmann, Yonatan Belinkov, Sharon Qian, Daniel Nevo, Yaron Singer, and Stuart Shieber.
\newblock Investigating gender bias in language models using causal mediation analysis.
\newblock \emph{Advances in neural information processing systems}, 33:\penalty0 12388--12401, 2020.

\bibitem[von Oswald et~al.(2022)von Oswald, Niklasson, Randazzo, Sacramento, Mordvintsev, Zhmoginov, and Vladymyrov]{von2022transformers}
Johannes von Oswald, Eyvind Niklasson, Ettore Randazzo, Jo{\~a}o Sacramento, Alexander Mordvintsev, Andrey Zhmoginov, and Max Vladymyrov.
\newblock Transformers learn in-context by gradient descent.
\newblock \emph{arXiv preprint arXiv:2212.07677}, 2022.

\bibitem[Wang et~al.(2022)Wang, Variengien, Conmy, Shlegeris, and Steinhardt]{wang2022interpretability}
Kevin Wang, Alexandre Variengien, Arthur Conmy, Buck Shlegeris, and Jacob Steinhardt.
\newblock Interpretability in the wild: a circuit for indirect object identification in gpt-2 small.
\newblock \emph{arXiv preprint arXiv:2211.00593}, 2022.

\bibitem[Wei et~al.(2022)Wei, Wang, Schuurmans, Bosma, Xia, Chi, Le, Zhou, et~al.]{wei2022chain}
Jason Wei, Xuezhi Wang, Dale Schuurmans, Maarten Bosma, Fei Xia, Ed~Chi, Quoc~V Le, Denny Zhou, et~al.
\newblock Chain-of-thought prompting elicits reasoning in large language models.
\newblock \emph{Advances in Neural Information Processing Systems}, 35:\penalty0 24824--24837, 2022.

\bibitem[Wei et~al.(2023)Wei, Wei, Tay, Tran, Webson, Lu, Chen, Liu, Huang, Zhou, et~al.]{wei2023larger}
Jerry Wei, Jason Wei, Yi~Tay, Dustin Tran, Albert Webson, Yifeng Lu, Xinyun Chen, Hanxiao Liu, Da~Huang, Denny Zhou, et~al.
\newblock Larger language models do in-context learning differently.
\newblock \emph{arXiv preprint arXiv:2303.03846}, 2023.

\bibitem[Wortsman et~al.(2023)Wortsman, Lee, Gilmer, and Kornblith]{wortsman2023replacing}
Mitchell Wortsman, Jaehoon Lee, Justin Gilmer, and Simon Kornblith.
\newblock Replacing softmax with relu in vision transformers.
\newblock \emph{arXiv preprint arXiv:2309.08586}, 2023.

\bibitem[Xie et~al.(2021)Xie, Raghunathan, Liang, and Ma]{xie2021explanation}
Sang~Michael Xie, Aditi Raghunathan, Percy Liang, and Tengyu Ma.
\newblock An explanation of in-context learning as implicit bayesian inference.
\newblock \emph{arXiv preprint arXiv:2111.02080}, 2021.

\bibitem[Zhang et~al.(2023{\natexlab{a}})Zhang, Frei, and Bartlett]{zhang2023trained}
Ruiqi Zhang, Spencer Frei, and Peter~L Bartlett.
\newblock Trained transformers learn linear models in-context.
\newblock \emph{arXiv preprint arXiv:2306.09927}, 2023{\natexlab{a}}.

\bibitem[Zhang et~al.(2023{\natexlab{b}})Zhang, Zhang, Yang, and Wang]{zhang2023and}
Yufeng Zhang, Fengzhuo Zhang, Zhuoran Yang, and Zhaoran Wang.
\newblock What and how does in-context learning learn? bayesian model averaging, parameterization, and generalization.
\newblock \emph{arXiv preprint arXiv:2305.19420}, 2023{\natexlab{b}}.

\bibitem[Zhao et~al.(2021)Zhao, Wallace, Feng, Klein, and Singh]{zhao2021calibrate}
Zihao Zhao, Eric Wallace, Shi Feng, Dan Klein, and Sameer Singh.
\newblock Calibrate before use: Improving few-shot performance of language models.
\newblock In \emph{International Conference on Machine Learning}, pages 12697--12706. PMLR, 2021.

\end{thebibliography}

\appendix
\makeatletter
\def\renewtheorem#1{%
  \expandafter\let\csname#1\endcsname\relax
  \expandafter\let\csname c@#1\endcsname\relax
  \gdef\renewtheorem@envname{#1}
  \renewtheorem@secpar
}
\def\renewtheorem@secpar{\@ifnextchar[{\renewtheorem@numberedlike}{\renewtheorem@nonumberedlike}}
\def\renewtheorem@numberedlike[#1]#2{\newtheorem{\renewtheorem@envname}[#1]{#2}}
\def\renewtheorem@nonumberedlike#1{  
\def\renewtheorem@caption{#1}
\edef\renewtheorem@nowithin{\noexpand\newtheorem{\renewtheorem@envname}{\renewtheorem@caption}}
\renewtheorem@thirdpar
}
\def\renewtheorem@thirdpar{\@ifnextchar[{\renewtheorem@within}{\renewtheorem@nowithin}}
\def\renewtheorem@within[#1]{\renewtheorem@nowithin[#1]}
\makeatother

\renewtheorem{theorem}{Theorem}[section]
\renewtheorem{lemma}[theorem]{Lemma}
\renewtheorem{remark}{Remark}
\renewtheorem{corollary}[theorem]{Corollary}
\renewtheorem{corollary*}{Corollary}
\renewtheorem{observation}[theorem]{Observation}
\renewtheorem{proposition}[theorem]{Proposition}
\renewtheorem{definition}[theorem]{Definition}
\renewtheorem{claim}[theorem]{Claim}
\renewtheorem{fact}[theorem]{Fact}
\renewtheorem{assumption}{Assumption}%
\renewcommand{\theassumption}{\Alph{assumption}}
\renewtheorem{conjecture}[theorem]{Conjecture}

\section{Technical tools}

The following convergence result for minimizing a smooth and strongly convex function is standard from the convex optimization literature, e.g. by adapting the learning rate in~\cite[Theorem 3.10]{bubeck2015convex} from $\eta=1/\beta$ to any $\eta\le 1/\beta$.

\begin{proposition}[Gradient descent for smooth and strongly convex functions]
\label{prop:strongly-convex-gd}
Suppose $L:\R^d\to\R$ is $\alpha$-strongly convex and $\beta$-smooth for some $0<\alpha\le\beta$. Then, the gradient descent iterates $\bw^{t+1}_\gd\defeq \bw^t_\gd - \eta\grad L(\bw^t_\gd)$ with learning rate $\eta\le 1/\beta$ and initialization $\bw^0_\gd\in\R^d$ satisfies for any $t\ge 1$,
\begin{align*}
    \ltwo{\bw^t_\gd - \bw^\star}^2 \le \exp\paren{-\eta\alpha\cdot t} \cdot \ltwo{\bw^0_\gd - \bw^\star}^2.
\end{align*}
where $\bw^\star\defeq \argmin_{\bw\in\R^d} L(\bw)$ is the minimizer of $L$.
\end{proposition}
\section{Proofs for Section~\ref{sec:fixed-rep}}
\label{app:proof-fixed-rep}

Throughout the rest of this and next section, we consider transformer architectures defined in~\cref{sec:prelim} where we choose $\barsig$ to be the (entry-wise) ReLU activation normalized by sequence length, following~\citep{bai2023transformers}: For all $\bA\in\R^{N\times N}$ and $i,j\in[N]$,
\begin{align}
\label{eqn:normalized-relu}
    \brac{\barsig(\bA)}_{ij} = \frac{1}{j}\sigma(A_{ij}),
\end{align}
where we recall $\sigma(t)=\max\sets{t,0}$ denotes the standard ReLU.
This activation is similar as the softmax in that, for every (query index) $j$, the resulting attention weights $\sets{\frac{1}{j}\sigma(A_{ij})}_{i\in[j]}$ is approximately a probability distribution in typical scenarios, in the sense that they are non-negative and sum to $O(1)$ when each $A_{ij}=O(1)$. We remark that transformers with (normalized) ReLU activation is recently shown to achieve comparable performance with softmax in larger-scale tasks~\citep{shen2023study,wortsman2023replacing}.

With activation chosen as~\cref{eqn:normalized-relu}, a (decoder-only) attention layer $\tbH=\Attn_\btheta(\bH)$ with $\btheta=(\bQ_m,\bK_m,\bV_m)_{m\in[M]}$ takes the following form in vector notation:
\begin{align*}
    \tbh_i = \bh_i + \sum_{m=1}^M \frac{1}{i}\sum_{j=1}^i \sigma\paren{\<\bQ_m\bh_i, \bK_m\bh_j\>} \cdot \bV_m\bh_j.
\end{align*}

Recall our input format~\cref{eqn:input-format}: 
\begin{align*}
\bH = \begin{bmatrix}
\bx_1 & \bzero & \dots & \bx_{N} & \bzero \\
0 & y_1 & \dots & 0 & y_N \\
\bp^x_1 & \bp^y_1 & \dots & \bp^x_N & \bp^y_N
\end{bmatrix}
\in \R^{\Dhid\times 2N}.
\end{align*}
We will use $(\bh_k)_{k\in[2N]}$ and $(\bh^x_i, \bh^y_i)_{i\in[N]}$ interchangeably to denote the tokens in~\cref{eqn:input-format}, where $\bh^x_i\defeq \bh_{2i-1}$ and $\bh^y_i\defeq \bh_{2i}$. Similarly, we will use $(\bp^x_i,\bp^y_i)_{i\in[N]}$ and $(\bp_k)_{k\in[2N]}$ interchangably to denote the positional encoding vectors in~\cref{eqn:input-format}, where $\bp_{2i-1}\defeq \bp^x_i$ and $\bp_{2i}\defeq \bp^y_i$. Unless otherwise specified, we typically reserve use $i,j$ as (query, key) indices within $[N]$ and $k,\ell$ as (query, key) indices within $[2N]$.

We use the following positional encoding vectors for all $i\in[N]$:
\begin{align}
\label{eqn:positional-encoding}
\begin{aligned}
    & \bp^x_i = [\bzero_{\Dhid-d-9}; 1; 2i-1; (2i-1)^2; (2i-1)^3; i; i^2; 1; i], \\
    & \bp^y_i = [\bzero_{\Dhid-d-9}; 1; 2i; (2i)^2; (2i)^3; i; i^2; 0; 0].
\end{aligned}
\end{align}
Note that $\bp_k$ contains $[1; k; k^2; k^3]$ for all $k\in[2N]$; $\bp^x_i$, $\bp^y_i$ contains $[i; i^2]$, an indicator of being an $x$-token, and the product of the indicator and $i$.

\subsection{Useful transformer constructions}

\begin{lemma}[Copying by a single attention head]
\label{lem:copy}
There exists a single-head attention layer $\btheta = (\bQ, \bK, \bV)\subset\R^{\Dhid\times \Dhid}$ that copies each $\bx_i$ into the next token for every input $\bH$ of the form~\cref{eqn:input-format}, i.e.
\begin{align*}
\Attn_\btheta(\bH) = \begin{bmatrix}
\bx_1 & \bx_1 & \dots & \bx_{N} & \bx_{N} \\
0 & y_1 & \dots & 0 & y_N \\
\bp^x_1 & \bp^y_1 & \dots & \bp^x_N & \bp^y_N
\end{bmatrix}
\in \R^{\Dhid\times 2N}.
\end{align*}
\end{lemma}
\begin{proof}
By assumption of the positional encoding vectors, we can define matrices $\bQ,\bK\in\R^{\Dhid\times\Dhid}$ such that for all $k,\ell\in[2N]$,
\begin{align*}
    \bQ\bh_k = [k^3; k^2; k; \bzero_{\Dhid-3}], \quad \bK\bh_\ell = [-1; 2\ell+2; -\ell^2-2\ell; \bzero_{\Dhid-3}].
\end{align*}
This gives that for all $\ell\le k$,
\begin{align*}
    & \quad \sigma\paren{ \<\bQ\bh_k, \bK\bh_\ell\> } \\
    & = \sigma\paren{ -k^3+k^2(2\ell+2)-k(\ell^2+2\ell) } = \sigma\paren{ k(1 - (k-\ell-1)^2) } = k\indic{\ell=k-1}.
\end{align*}
Further defining $\bV$ such that $\bV\bh^x_i=[\bx_i; \bzero]$ and $\bV\bh^y_i=\bzero$, we have for every $k\in[2N]$ that
\begin{align*}
    & \quad \sum_{\ell\le k} \frac{1}{k}\sigma\paren{ \<\bQ\bh_k, \bK\bh_\ell\> } \bV\bh_{\ell} \\
    & = \frac{1}{k} \cdot k\indic{\ell=k-1}\cdot [\bx_{\ceil{\ell/2}}\indic{\ell~\textrm{is odd}}; \bzero] = [\bx_{\ceil{\ell/2}}; \bzero] \cdot \indic{\ell=k-1~\textrm{and}~\ell~\textrm{is odd}}.
\end{align*}
By the residual structure of the attention layer, the above exactly gives the desired copying behavior, where every $\bx_i$ on the odd token $\bH$ is copied to the next token.
\end{proof}

\begin{lemma}[Linear prediction layer]
\label{lem:linear-prediction}
For any $\Bx,\Bw,\By>0$, there exists an attention layer $\btheta=\sets{(\bQ_m,\bK_m,\bV_m)}_{m\in[M]}$ with $M=2$ heads such that the following holds. For any input sequence $\bH\in\R^{\Dhid\times 2N}$ that takes form 
\begin{align*}
    \bh^x_{i}=[\bx_i; 0; \bw_i; \bp^x_i], \quad \bh^y_i=[\bx_i; y_i; \bzero_d; \bp^y_i]
\end{align*}
with $\ltwo{\bx_i}\le \Bx$, $\abs{y_i}\le \By$, and $\ltwo{\bw}\leq\Bw$, it gives output $\Attn_\btheta(\bH)=\tbH\in\R^{\Dhid\times 2N}$ with
\begin{align*}
    \tbh^x_i = \tbh_{2i-1}=[\bx_i; \hy_i; \bw_i; \bp^x_i], \quad \textrm{where}~~\hy_i = \<\bx_i, \bw_i\>
\end{align*}
for all $i\in[N]$.
\end{lemma}
\begin{proof}
Let $R\defeq \max\sets{\Bx\Bw, \By}$. Define matrices $(\bQ_m, \bK_m, \bV_m)_{m=1,2}$ as
\begin{align*}
    \bQ_1\bh^x_i = \begin{bmatrix}
        \bw_i \\ i \\ R \\ \bzero
    \end{bmatrix}, \quad 
    \bK_1\bh^x_j = \bK_1\bh^y_j = \begin{bmatrix}
        \bx_j \\ -2R \\ 2j+1 \\ \bzero
    \end{bmatrix},
    \bV_1\bh_\ell = \begin{bmatrix}
        \bzero_d \\ \ell \\ \bzero_{\Dhid-d-1}
    \end{bmatrix}, \\
    \bQ_2\bh^x_i = \begin{bmatrix}
        i \\ R \\ \bzero
    \end{bmatrix}, \quad 
    \bK_2\bh^x_j = \bK_1\bh^y_j = \begin{bmatrix}
        -2R \\ 2j+1 \\ \bzero
    \end{bmatrix},
    \bV_2\bh_\ell = -\begin{bmatrix}
        \bzero_d \\ \ell \\ \bzero_{\Dhid-d-1}
    \end{bmatrix}
\end{align*}
for all $i,j\in[N]$ and $\ell\in[2N]$. 
For every $i\in[N]$, we then have
\begin{align*}
    & \quad \sum_{m=1}^2 \sum_{\ell=1}^{2i-1}\frac{1}{2i-1} \sigma\paren{ \<\bQ_m\bh^x_i, \bK_m\bh_\ell\> } \cdot \bV_m\bh_\ell \\
    & = \frac{1}{2i-1} \bigg( \sum_{j=1}^i \brac{\sigma\paren{ \bw_i^\top\bx_j + R(-2i+2j+1) } - \sigma\paren{ R(-2i+2j+1) }} \cdot [\bzero_d; 2j-1; \bzero_{\Dhid-d-1}] \\
    & \qquad + \sum_{j=1}^{i-1} \brac{\sigma\paren{ \bw_i^\top\bx_j + R(-2i+2j-1) } - \sigma\paren{ R(-2i+2j+1) }} \cdot [\bzero_d; 2j; \bzero_{\Dhid-d-1}] \bigg) \\
    & = \frac{1}{2i-1} \cdot \bw_i^\top\bx_i \cdot [\bzero_d; 2i-1; \bzero_{\Dhid-d-1}] = [\bzero_d; \bw_i^\top\bx_i; \bzero_{\Dhid-d-1}].
\end{align*}
By the residual structure of an attention layer, the above shows the desired result.
\end{proof}

\begin{lemma}[Implementing MLP representation by transformers]
\label{lem:mlp}
Fix any MLP representation function $\Phi^\star$ of the form~\cref{eqn:mlp}, suppose $\Dhid\ge\max\sets{2D,D+d+10}$, where $D$ is the hidden dimension within the MLP~\cref{eqn:mlp}. Then there exists a transformer $\TF_\btheta$ with $(L+1)$ layers and $5$ heads that exactly implements $\Phi^\star$ in a token-wise fashion, i.e. for any input $\bH$ of form~\cref{eqn:input-format},
\begin{align*}
    \tbH = \TF_\btheta(\bH) = \begin{bmatrix}
    \Phi^\star(\bx_1) & \bzero & \dots & \Phi^\star(\bx_N) & \bzero \\
    0 & y_1 & \dots & 0 & y_N \\
    \tbp^x_1 & \tbp^y_1 & \dots & \tbp^x_{N} & \tbp^y_{N}
    \end{bmatrix},
\end{align*}
where $\tbp^x_i,\tbp^y_i$ differs from $\bp^x_i,\bp^y_i$ only in the dimension of their zero paddings.
\end{lemma}
\begin{proof}
Recall that $\Phi^\star(\bx)=\sigma_\rho(\bB_L^\star\cdots\sigma_\rho(\bB_1^\star\bx)\cdots)$.
We first show how to implement a single MLP layer $\bx\mapsto \sigma_\rho(\bB_1^\star\bx)$ by an (MLP-Attention) structure. 

Consider any input token $\bh^x_i=[\bx_i; 0; \bp^x_i]$ at an $x$-location. Define matrices $\bW_1,\bW_2\in\R^{\Dhid\times\Dhid}$ such that
\begin{align*}
    & \bW_1\bh^x_i = \begin{bmatrix}
        \bB_1^\star\bx_i \\ -\bB_1^\star\bx_i \\ \bzero
    \end{bmatrix}, \quad 
    \sigma\paren{ \bW_1\bh^x_i } = \begin{bmatrix}
        \sigma(\bB_1^\star\bx_i) \\ \sigma(-\bB_1^\star\bx_i) \\ \bzero
    \end{bmatrix}, \\
    & \bW_2\sigma\paren{ \bW_1\bh^x_i } = \begin{bmatrix}
        \bzero_d \\ \sigma(\bB_1^\star\bx_i) - \rho\sigma(-\bB_1^\star\bx_i) \\ \bzero
    \end{bmatrix} = \begin{bmatrix}
        \bzero_d \\ \sigma_\rho(\bB_1^\star\bx_i) \\ \bzero
    \end{bmatrix}.
\end{align*}
Therefore, the MLP layer $(\bW_1,\bW_2)$ outputs
\begin{align*}
    \barbh^x_i \defeq \brac{\MLP_{\bW_1,\bW_2}(\bH)}^x_i = \bh^x_i + \bW_2\sigma(\bW_1\bh^x_i) = \begin{bmatrix}
        \bx_i \\ \sigma_\rho(\bB_1^\star\bx_i) \\ \bzero \\ \bp^x_i
    \end{bmatrix},
\end{align*}
and does not change the $y$-tokens. 

We next define an attention layer that ``moves'' $\sigma_\rho(\bB_1\bx_i)$ to the beginning of the token, and removes $\bx_i$. Define three attention heads $\btheta=(\bQ_m,\bK_m,\bV_m)_{m\in[3]}$ as follows:
\begin{align*}
    & \bQ_{\sets{1,2,3}}\barbh_k = \begin{bmatrix}
        k^2 \\ k \\ k\indic{k~\textrm{is odd}} \\ \bzero
    \end{bmatrix},
    \bK_{\sets{1,2,3}}\barbh_\ell = \begin{bmatrix}
        -1 \\ \ell \\ 1 \\ \bzero
    \end{bmatrix}, \\
    & \bV_1\barbh^x_j = \begin{bmatrix}
        \sigma_\rho(\bB_1^\star\bx_j) \\ \bzero_d \\ \bzero 
    \end{bmatrix},
    \bV_2\barbh^x_j = \begin{bmatrix}
        -\bx_j \\ \bzero_D \\ \bzero 
    \end{bmatrix},
    \bV_3\barbh^x_j = \begin{bmatrix}
        \bzero_d \\ -\sigma_\rho(\bB_1^\star\bx_j) \\ \bzero 
    \end{bmatrix}.
\end{align*}
The values for $\bV_{1,2,3}\barbh^y_i$ are defined automatically by the same operations over the $\barbh^y_i$ tokens (which does not matter to the proof, as we see shortly). For any $\ell\le k$ and $m\in[3]$,
\begin{align*}
    \frac{1}{k}\sigma\paren{ \<\bQ_m\barbh_k, \bK_m\barbh_\ell\> } = \frac{1}{k}\sigma\paren{ k(-k+\ell+\indic{k~\textrm{is odd}}) } = \indic{\ell=k,~k~\textrm{is odd}}.
\end{align*}
Therefore, these three attention heads are only active iff the query token $k=2i-1$ is odd (i.e. being an $x$-token) and $\ell=k=2i-1$. At such tokens, the three value matrices (combined with the residual structure of attention) would further remove the $\bx_i$ part, and move $\sigma_\rho(\bB_1^\star\bx_i)$ to the beginning of the token, i.e.
\begin{align*}
    \tbh^x_i = \brac{\Attn_{\btheta}(\barbH)}^x_i = \begin{bmatrix}
        \sigma_\rho(\bB_1^\star\bx_i) \\ 0 \\ \bp^x_i
    \end{bmatrix},
\end{align*}
and $\tbh^y_i=\bh^y_i$. Additionally, we now add two more attention heads into $\btheta$ to move all $y_i$ from entry $d+1$ to $D+1$, and leaves the $x$-tokens unchanged.

Repeating the above argument $L$ times, we obtain a structure (MLP-Attention-$\dots$-MLP-Attention) with five heads in each attention layer that exactly implements the $\Phi^\star$ in a token-wise fashion. This structure can be rewritten as an $(L+1)$-layer transformer by appending an identity \{Attention, MLP\} layer (with zero weights) \{before, after\} the structure respectively, which completes the proof.
\end{proof}

\subsection{In-context ridge regression by decoder transformer}

This section proves the existence of a decoder transformer that approximately implements in-context ridge regression at every token $i\in[N]$ simultaneously. For simplicity, we specialize our results to the ridge regression problem; however, our construction can be directly generalized to any (generalized) linear models with a sufficiently smooth loss, by approximating the gradient of the loss by sum of relus~\citep[Section 3.5]{bai2023transformers}.

Denote the regularized empirical risk for ridge regression on dataset $\cD_i=\sets{(\bx_j, y_j)}_{j\in[i]}$ by
\begin{align}
\label{eqn:ridge-risk-i}
    \Lilam(\bw)\defeq \frac{1}{2i}\sum_{j=1}^i \paren{\bw^\top\bx_j -  y_j}^2 + \frac{\lambda}{2}\ltwo{\bw}^2
\end{align}
for all $i\in[N]$. Let $\hbwilam\defeq \argmin_{\bw\in\R^d} \Limlam(\bw)$ denote the minimizer of the above risk (solution of ridge regression) for dataset $\cD_{i-1}$. We further understand $\hat{L}_0^\lambda(\bw)\defeq 0$ and $\hbw_1^\lambda\defeq \bzero$. Let $\Li(\bw)\defeq \Li^{0}(\bw)$ denote the unregularized version of the above risk. 

\begin{proposition}[Approximating a single GD step by a single attention layer]
\label{prop:one-step-gd}
For any $\eta>0$ and any $\Bx,\Bw,\By>0$, there exists an attention layer $\btheta=\sets{(\bQ_m,\bK_m,\bV_m)}_{m\in[M]}$ with $M=3$ heads such that the following holds. For any input sequence $\bH\in\R^{\Dhid\times 2N}$ that takes form 
\begin{align*}
    \bh^x_{i}=[\bx_i; 0; \bw_i; \bp^x_i], \quad \bh^y_i=[\bx_i; y_i; \bzero_d; \bp^y_i]
\end{align*}
with $\ltwo{\bx_i}\le \Bx$, $\abs{y_i}\le \By$, and $\ltwo{\bw}\leq\Bw$,  it gives output $\Attn_\btheta(\bH)=\tbH\in\R^{\Dhid\times 2N}$ with $\tbh^x_i = \tbh_{2i-1}=[\bx_i; 0; \wt{\bw}_i; \bp^x_i]$, where
\begin{align*}
    \wt{\bw}_i = \bw_i - \eta_i\grad\Limlam(\bw_i)
\end{align*}
with $\eta_i=\frac{i-1}{2i-1}\eta$, and $\tbh^y_i=\bh^y_i$, for all $i\in[N]$.
\end{proposition}
\begin{proof}
Let $R\defeq \max\sets{\Bx\Bw, \By}$.
By the form of the input $(\bh_k)_{k\in[2N]}$ in~\cref{eqn:input-format}, we can define two attention heads $\sets{(\bQ_m,\bK_m,\bV_m)}_{m=1,2}\subset \R^{\Dhid\times \Dhid}$ such that for all $i,j\in[N]$,
\begin{align*}
    & \bQ_1\bh^x_i = \begin{bmatrix} \bw_i/2 \\ -1 \\ i \\ -3R \\ -R \\ \bzero \end{bmatrix}, \quad
    \bK_1\bh^y_j = \begin{bmatrix} \bx_j \\ y_j \\ 3R \\ j \\ 1 \\ \bzero \end{bmatrix}, \quad
    \bV_1\bh_j^x = \bV_1\bh_j^y = -\eta\cdot \begin{bmatrix} \bzero_{d+1} \\ \bx_j \\ \bzero_{\Dhid-2d-1} \end{bmatrix}, \\
    & \bQ_2\bh^x_i = \bQ_2\bh^y_i = \begin{bmatrix} i \\ -3R \\ -R \\ \bzero \end{bmatrix}, \quad
    \bK_2\bh^x_j = \bK_2\bh^y_j = \begin{bmatrix} 3R \\ j \\ 1 \\ \bzero \end{bmatrix}, \quad
    \bV_2\bh^x_j = \bV_2\bh^y_j = \eta\cdot \begin{bmatrix} \bzero_{d+1} \\ \bx_j \\ \bzero_{\Dhid-2d-1} \end{bmatrix}.
\end{align*}
Further, $\bQ_1\bh^y_i$ takes the same form as $\bQ_1\bh^x_i$ except for replacing the $\bw_i/2$ location with $\bzero_d$ and replacing the $-1$ location with $0$ (using the indicator for being an $x$-token within $\bp^x_i,\bp^y_i$); $\bK_1\bh^x_j$ takes the same form as $\bK_1\bh^y_j$ except for replacing the $y_j$ location with $0$.

Fixing any $i\in[N]$. We have for all $j\le i-1$,
\begin{align*}
    & \sigma\paren{\<\bQ_1\bh^x_i, \bK_1\bh^y_j\>} - \sigma\paren{\<\bQ_2\bh^x_i, \bK_2\bh^y_j\>} \\
    =&~ \sigma\paren{\bw_i^\top\bx_j/2 - y_j + R(3i-3j-1)} - \sigma\paren{ R(3i-3j-1)} = \bw_i^\top\bx_j/2 - y_j,
\end{align*}
and for all $j\le i$,
\begin{align*}
    & \sigma\paren{\<\bQ_1\bh^x_i, \bK_1\bh^x_j\>} - \sigma\paren{\<\bQ_2\bh^x_i, \bK_2\bh^x_j\>} \\
    =&~ \sigma\paren{\bw_i^\top\bx_j/2 + R(3i-3j-1)} - \sigma\paren{ R(3i-3j-1)} = \bw_i^\top\bx_j/2 \cdot \indic{j\le i-1}.
\end{align*}
Above, we have used $\abs{\bw_i^\top\bx_j/2-y_j}\le 3R/2$, $\abs{\bw_i^\top\bx_j/2}\le R/2$, and the fact that $\sigma(z+M)-\sigma(M)$ equals $z$ for $M\ge \abs{z}$ and $0$ for $M\le -\abs{z}$.

Therefore for all $j\le i-1$,
\begin{align*}
    & \quad \sigma\paren{\<\bQ_1\bh^x_i, \bK_1\bh^y_j\>} \bV_1\bh^y_j + \sigma\paren{\<\bQ_2\bh^x_i, \bK_2\bh^y_j\>} \bV_2\bh^y_j \\
    & = \paren{ \sigma\paren{\<\bQ_1\bh^x_i, \bK_1\bh^y_j\>} - \sigma\paren{\<\bQ_2\bh^x_i, \bK_2\bh^y_j\>} } \cdot -\eta [\bzero_{d+1}; \bx_j; \bzero_{\Dhid-2d-1}] \\
    & = -\eta\paren{\bw_i^\top\bx_j / 2 - y_j} \cdot \brac{\bzero_{d+1}; \bx_j; \bzero_{\Dhid-2d-1}},
\end{align*}
and similarly for all $j\le i$,
\begin{align*}
    & \quad \sigma\paren{\<\bQ_1\bh^x_i, \bK_1\bh^x_j\>} \bV_1\bh^x_j + \sigma\paren{\<\bQ_2\bh^x_i, \bK_2\bh^x_j\>} \bV_2\bh^x_j \\
    & = -\eta\paren{\bw_i^\top\bx_j / 2} \indic{j\le i-1} \cdot \brac{\bzero_{d+1}; \bx_j; \bzero_{\Dhid-2d-1}}
\end{align*}
Summing the above over all key tokens $\ell\in[2i-1]$, we obtain the combined output of the two heads at query token $2i-1$ (i.e. the $i$-th $x$-token):
\begin{align}
\label{eqn:one-step-gd-12}
\begin{aligned}
    & \quad \sum_{\ell=1}^{2i-1} \sum_{m=1,2} \frac{1}{2i-1} \sigma\paren{ \<\bQ_m\bh_{2i-1}, \bK_m\bh_\ell\> } \bV_m\bh_\ell \\
    & = \sum_{j=1}^{i-1} \sum_{m=1,2} \frac{1}{2i-1}\sigma\paren{ \<\bQ_m\bh^x_i, \bK_m\bh^y_j\> } \bV_m\bh^y_j + \sum_{j=1}^{i} \sum_{m=1,2} \frac{1}{2i-1}\sigma\paren{ \<\bQ_m\bh^x_i, \bK_m\bh^x_j\> } \bV_m\bh^x_j \\
    & = \frac{1}{2i-1}\brac{ \sum_{j=1}^{i-1} -\eta\paren{\bw_i^\top\bx_j / 2 - y_j} + \sum_{j=1}^i -\eta\paren{\bw_i^\top\bx_j / 2} \indic{j\le i-1} } \cdot \brac{\bzero_{d+1}; \bx_j; \bzero_{\Dhid-2d-1}} \\
    & = \frac{i-1}{2i-1} \cdot \brac{\bzero_{d+1}; -\eta\grad \Lim(\bw_i); \bzero_{\Dhid-2d-1}}.
\end{aligned}
\end{align}
It is straightforward to see that, repeating the same operation at query token $2i$ (i.e. the $i$-th $y$-token) would output $\bzero_{\Dhid}$, since the query vector $\bQ_1\bh_i^y$ contains $[\bzero_d; 0]$ instead of $[\bw_i/2; -1]$ as in $\bQ_1\bh_i^x$.

We now define one more attention head $(\bQ_3, \bK_3, \bV_3)\subset \R^{D\times D}$ such that for all $k\in[2N]$, $j\in[N]$,
\begin{align*}
\bQ_3\bh_k = \begin{bmatrix} k^2 \\ k \\ 1 \\ \bzero \end{bmatrix}, \quad
\bK_3\bh_\ell = \begin{bmatrix} -1/2 \\ (1-\ell)/2 \\ 1-\ell/2 \\ \bzero \end{bmatrix}, \quad \bV_3\bh^x_j = \begin{bmatrix} \bzero_{d+1} \\ -\eta\lambda \bw_j \\ \bzero_{\Dhid-2d-1} \end{bmatrix}, \quad
\bV_3\bh^y_j = \bzero_{\Dhid}.
\end{align*}
For any $\ell\le k$, we have
\begin{align*}
    \sigma\paren{ \<\bQ_3\bh_k, \bK_3\bh_\ell\> } = \sigma\paren{-k^2/2 + k(1-\ell)/2+1-\ell/2} = \frac{k-1}{2}\sigma(-k+\ell+1) = \frac{k-1}{2}\indic{\ell=k}.
\end{align*}
Therefore, for query token $k=2i-1$, the attention head outputs
\begin{align}
\label{eqn:one-step-gd-3}
\begin{aligned}
& \quad \sum_{\ell=1}^{k}\frac{1}{k} \sigma\paren{ \<\bQ_3\bh_k, \bK_3\bh_\ell\> } \bV_m\bh_\ell = \sum_{\ell=1}^{k}\frac{1}{k} \cdot \frac{k-1}{2}\indic{\ell=k} \cdot \bV_m\bh_\ell \\
& = \frac{k-1}{2k}\cdot \bV_m\bh_k = \frac{i-1}{2i-1}\cdot \bV_m\bh^x_i = \frac{i-1}{2i-1} \cdot \brac{ \bzero_{d+1}; -\eta \lambda\bw_i; \bzero_{\Dhid-2d-1} }.
\end{aligned}
\end{align}
It is straightforward to see that the same attention head at query token $k=2i$ outputs $\bzero_{\Dhid}$, as the value vector $\bV_3\bh_k=\bV_3\bh^y_i$ is zero.

Combining~\cref{eqn:one-step-gd-12} and~\cref{eqn:one-step-gd-3}, letting the full attention layer $\btheta\defeq \sets{(\bQ_m,\bK_m,\bV_m)}_{m=1,2,3}$, we have $\Attn_\btheta(\bH)\\=\tbH$, where for all $i\in[N]$,
\begin{align*}
    & \quad \tbh^x_i = \tbh_{2i-1} = \bh_{2i-1} + \sum_{m=1}^3 \sum_{\ell=1}^{2i-1} \frac{1}{2i-1}\sigma\paren{ \<\bQ_m\bh_{2i-1}, \bK_m\bh_\ell\> } \cdot \bV_m\bh_{\ell} \\
    & = 
    \begin{bmatrix}
    \bx_i \\ 0 \\ \bw_i \\ * 
    \end{bmatrix} + \frac{i-1}{2i-1} \begin{bmatrix} 
    \bzero_{d+1} \\ -\eta \paren{\grad \Lim(\bw_i) + \lambda\bw_i} \\ \bzero_{\Dhid-2d-1} 
    \end{bmatrix} = \begin{bmatrix} 
    \bx_i \\ 0 \\ \bw_i -\eta_i \grad \Limlam(\bw_i) \\ *
    \end{bmatrix},
\end{align*}
where $\eta_i\defeq \frac{i-1}{2i-1}\bw_i$, and $\tbh^y_i=\bh^y_i$.
This finishes the proof.
\end{proof}

\begin{theorem}[In-context ridge regression by decoder-only transformer]
\label{thm:ridge}
For any $\lambda\ge 0$, $\Bx,\Bw,\By>0$ with $\kappa\defeq 1+\Bx^2/\lambda$, and $\eps<\Bx\Bw/2$, let $\Dhid\ge 2d+10$, then there exists an $L$-layer transformer $\TF_\btheta$ with $M=3$ heads and hidden dimension $\Dhid$, where
\begin{align}
\textstyle
\label{eqn:ridge-capacity-bound}
L=\ceil{3\kappa\log(\Bx\Bw/(2\eps))} + 2,
\end{align}
such that the following holds. On any input matrix $\bH$ of form~\cref{eqn:input-format} such that problem~\cref{eqn:ridge-risk-i} has bounded inputs and solution: for all $i\in[N]$
\begin{align}\label{eqn:well-conditioned}
    \ltwo{\bx_i} \le \Bx, \quad \abs{y_i}\le \By, \quad \ltwo{\hbwilam}\le \Bw/2,
\end{align}
$\TF_\btheta$ approximately implements the ridge regression algorithm (minimizer of risk~\cref{eqn:ridge-risk-i}) at every token $i\in[N]$: The prediction $\hat{y}_i \defeq \brac{\TF_\btheta(\bH)}_{d+1,2i-1}$ satisfies 
\begin{align}
    \abs{ \hy_i - \<\hbwilam, \bx_i\> } \le \eps.
\end{align}
\end{theorem}
\begin{proof}
The proof consists of two steps.

\paragraph{Step 1}
We analyze the convergence rate of gradient descent on $\Limlam$ simultaneously for all $2\le i\le N$, each with learning rate $\eta_i=\frac{i-1}{2i-1}\eta$ as implemented in~\cref{prop:one-step-gd}.

Fix $2\le i\le N$. Consider the ridge risk $\Limlam$ defined in~\cref{eqn:ridge-risk-i}, which is a convex quadratic function that is $\lambda$-strongly convex and $\lammax\paren{\bX_{i-1}^\top\bX_{i-1}/(i-1)}+\lambda\le \Bx^2+\lambda\eqdef \beta$ smooth over $\R^d$. Recall $\kappa=\beta/\lambda=1+\Bx^2/\lambda$. 

Consider the following gradient descent algorithm on $\Limlam$: Initialize $\bw^0_i\defeq \bzero$, and for every $t\ge 0$
\begin{align}
\label{eqn:ridge-gd-i}
    \bw^{t+1}_i = \bw^t_i - \eta_i\grad \Limlam(\bw^t_i),
\end{align}
with $\eta_i=\frac{i-1}{2i-1}\eta$. Taking $\eta\defeq 2/\beta$, we have $\eta_i\in [2/(3\beta), 1/\beta]$, and thus $\eta_i\lambda \in [2/(3\kappa), 1/\kappa]$.

By standard convergence results for strongly convex and smooth functions (\cref{prop:strongly-convex-gd}), we have for all $t\ge 1$ that
\begin{align*}
    \ltwo{\bw_i^t - \hbwilam}^2 \le \exp\paren{-\eta_i\lambda t}\ltwo{\bw_i^0 - \hbwilam}^2 = \exp\paren{-\eta_i\lambda t}\ltwo{\hbwilam}^2.
\end{align*}
Further, taking the number of steps as 
\begin{align*}
    T \defeq \ceil{3\kappa\log\paren{\frac{\Bx\Bw}{2\eps}}}
\end{align*}
so that $\eta_i\lambda T/2\ge 2/(3\kappa)\cdot 3\kappa\log(\Bx\Bw/(2\eps)) /2 = \log(\Bx\Bw/(2\eps))$, we have
\begin{align}
\label{eqn:ridge-gd-ridge}
    \ltwo{\bw_i^T - \hbwilam} \le \exp\paren{-\eta_i\lambda T/2}\ltwo{\hbwilam} \le \frac{2\eps}{\Bx\Bw} \cdot \frac{\Bw}{2} \le \frac{\eps}{\Bx}.
\end{align}

\paragraph{Step 2}
We construct a $(T+2)$-layer transformer $\TF_\btheta$ by concatenating the copying layer in~\cref{lem:copy}, $T$ identical gradient descent layers as constructed in~\cref{prop:one-step-gd}, and the linear prediction layer in~\cref{lem:linear-prediction}. Note that the transformer is attention only (all MLP layers being zero), and the number of heads within all layers is at most $3$.

The copying layer ensures that the output format is compatible with the input format required in ~\cref{prop:one-step-gd}, which in turn ensures that the $T$ gradient descent layers implement~\cref{eqn:ridge-gd-i} simultaneously for all $1\le i\le N$ ($\bw_1^T\defeq \bzero$ is not updated at token $i=1$). 
Therefore, the final linear prediction layer ensures that, the output matrix $\tbH\defeq \TF_\btheta(\bH)$ contains the following prediction at every $i\in[N]$:
\begin{align*}
    \hy_{i} \defeq [\tbh^x_i]_{d+1} = \<\bw_i^T, \bx_i\>,
\end{align*}
which satisfies
\begin{align*}
    \abs{\hy_i - \<\hbwilam, \bx_i\>} = \abs{ \<\bw_i^T - \hbwilam, \bx_i\> } \le (\eps/\Bx) \cdot \Bx = \eps.
\end{align*}
This finishes the proof.

\end{proof}

\subsection{Proof of Theorem~\ref{thm:fixed-rep}}
\label{app:proof-fixed-rep-main-theorem}

The result follows directly by concatenating the following two transformer constructions:
\begin{itemize}[leftmargin=1.5em]
    \item The MLP implementation module in~\cref{lem:mlp}, which has $(L+1)$-layers, 5 heads, and transforms every $\bx_i$ to $\Phi^\star(\bx_i)$ to give output matrix~\cref{eqn:transformed-input};
    \item The in-context ridge regression module in~\cref{thm:ridge} (with inputs being $\sets{\Phi^\star(\bx_i)}$ instead of $\bx_i$) which has $\cO\paren{\kappa\log(\Bphi\Bw/\eps)}$ layers, $3$ heads, and outputs prediction $\hy_i\defeq [\tbh^x_i]_{D+1}$ where $|\hy_i - \langle \Phi^\star(\bx_i), \hbwiphilam\rangle|\le \eps$, where $\hbwiphilam$ is the~\cref{eqn:phi-ridge} predictor.
\end{itemize}
Claim~\cref{eqn:transformed-input} can be seen by concatenating the $(L+1)$-layer MLP module with the first layer in the ridge regression module (\cref{thm:ridge}), which copies the $\Phi^\star(\bx_i)$ in each $x$ token to the same location in the succeeding $y$ token.

Further, the hidden dimension requirements are $\Dhid \ge \max\sets{2D, D+d+10}$ for the first module and $\Dhid\ge 2D+10$ for the second module, which is satisfied at our precondition $\Dhid=2D+d+10$. This finishes the proof.
\qed

\section{Proofs for Section~\ref{sec:dynamical-system}}

Recall our input format~\cref{eqn:input-format-dynamical-system} for the dynamical system setting:
\begin{align*}
    \bH \defeq \begin{bmatrix}
        \bx_1 & \dots & \bx_N \\
        \bp_1 & \dots & \bp_N
    \end{bmatrix} \in \R^{\Dhid\times N},
\end{align*}
our choice of the positional encoding vectors $\bp_i=[\bzero_{\Dhid-d-4}; 1; i; i^2; i^3]$ for all $i\in[N]$, and that we understand $\bx_i\defeq \bzero$ for all $i\le 0$.

\subsection{Useful transformer constructions}

\begin{lemma}[Copying for dynamical systems]
\label{Lem:copy-dynamic}
Suppose $\Dhid \ge kd+4$. For any $k\in[N]$, there exists a $(k+1)$-head attention layer $\btheta = \sets{(\bQ_m, \bK_m, \bV_m)}_{m\in[k+1]}\subset\R^{\Dhid\times \Dhid}$ such that for every input $\bH$ of the form~\cref{eqn:input-format-dynamical-system}, we have
\begin{align}
\label{eqn:input-format-after-copy}
    \tbH = \Attn_\btheta(\bH) = \begin{bmatrix}
    \bx_{1-k+1} & \dots & \bx_{i-k+1} & \dots & \bx_{N-k+1} \\
    \vert &  & \vert &  & \vert \\
    \bx_{1} & \dots & \bx_{i} & \dots & \bx_{N} \\
    \barbp_{1} & \dots & \barbp_{i} & \dots & \barbp_{N}
\end{bmatrix}
\in \R^{\Dhid\times N},
\end{align}
where $\barbp_i$ only differs from $\bp_i$ in the dimension of the zero paddings.
In words, $\Attn_\btheta$ copies the $k-1$ previous tokens $[\bx_{i-k+1}; \dots; \bx_{i-1}]$ onto the $i$-th token.
\end{lemma}
\begin{proof}
For every $k'\in[k]$, we define an attention head $(\bQ_{k'}, \bK_{k'}, \bV_{k'})\subset\R^{\Dhid\times\Dhid}$ such that for all $j\le i\in[N]$,
\begin{align*}
    & \bQ_{k'}\bh_i = [i^3; i^2; i; \bzero_{\Dhid-3}], \\
    & \bK_{k'}\bh_j = [-1; 2j+2(k'-1); -j^2+2(k'-1)j+1-(k'-1')^2; \bzero_{\Dhid-3}], \\
    & \bV_{k'}\bh_j = [\bzero_{(k-k')D}; \bx_j; \bzero].
\end{align*}
Note that
\begin{align*}
    & \quad \sigma\paren{ \<\bQ_{k'}\bh_i, \bK_{k'}\bh_j\> } = \sigma\paren{-i^3+2i^2j + 2(k'-1)i^2-ij^2+2ij(k'-1)+i-i(k'-1)^2} \\
    & = i\sigma\paren{1-(j-i+k'-1)^2} = i\indic{j=i-k'+1}.
\end{align*}
Therefore, at output token $i\in[N]$, this attention head gives
\begin{align*}
    \frac{1}{i}\sum_{j=1}^i \sigma\paren{ \<\bQ_{k'}\bh_i, \bK_{k'}\bh_j\> } \bV_{k'}\bh_j = \frac{1}{i}\cdot i \cdot \bV_{k'}\bh_{i-k'+1} = [\bzero_{(k-k')D}; \bx_{i-k'+1}; \bzero]
\end{align*}
when $i-k'+1\ge 1$, and zero otherwise. Combining all $k$ heads, and defining one more head $(\bQ_{k+1}, \bK_{k+1}, \bV_{k+1})$ to ``remove'' $\bx_i$ at its original location (similar as in the proof of~\cref{lem:mlp}), we have
\begin{align*}
    \sum_{m=1}^{k+1}\frac{1}{i}\sum_{j=1}^i \sigma\paren{ \<\bQ_{k'}\bh_i, \bK_{k'}\bh_j\> } \bV_{k'}\bh_j = \begin{bmatrix}
        \bx_{i-k+1} - \bx_i \\ \bx_{i-(k-1)+1} \\ \vert \\ \bx_i \\ \bzero
    \end{bmatrix}.
\end{align*}
By the residual structure of an attention layer, we have
\begin{align*}
    \brac{\Attn_\btheta(\bH)}_i = \begin{bmatrix}
        \bx_i \\ \bp_i
    \end{bmatrix} + \begin{bmatrix}
        \bx_{i-k+1} - \bx_i \\ \bx_{i-(k-1)+1} \\ \vert \\ \bx_i \\ \bzero
    \end{bmatrix} = 
    \begin{bmatrix}
        \bx_{i-k+1} \\ \bx_{i-(k-1)+1} \\ \vert \\ \bx_i \\ \barbp_i
    \end{bmatrix}.
\end{align*}
(The precondition $\Dhid\ge D+4$ guarantees that the $x$ entries would not interfere with the non-zero entries within $\bp_i$.) This is the desired result.
\end{proof}

\begin{lemma}[Implementing MLP representation for dynamical systems]
\label{lem:mlp-dynamical-system}
Fix any MLP representation function $\Phi^\star:\R^{kd}\to\R^D$ of the form~\cref{eqn:mlp}, suppose $\Dhid\ge 2(k+1)d+3D+2d+5$. Then there exists a module MLP-(Attention-MLP-$\dots$-Attention-MLP) with $L+1$ (Attention-MLP) blocks (i.e. transformer layers) and $5$ heads in each attention layer (this is equivalent to an $(L+2)$-layer transformer without the initial attention layer) that implements $\Phi^\star$ in the following fashion: For any input $\bH$ of form
\begin{align*}
    \bH = \begin{bmatrix}
    \barbx_1 & \dots & \barbx_N \\
    \barbp_{1} & \dots & \barbp_{N}
\end{bmatrix}
\end{align*}
where we recall $\barbx_i=[\bx_{i-k+1}; \dots; \bx_i]\in\R^{kd}$, the following holds. The first MLP layer outputs
\begin{align*}
\MLP^{(1)}(\bH) = \begin{bmatrix}
        \sigma_\rho(\bB_1^\star\barbx_1) & \dots & \sigma_\rho(\bB_1^\star\barbx_i) \\
        \bx_1 & \dots & \bx_i \\
        \barbp_1' & \dots & \barbp_i'
    \end{bmatrix}.
\end{align*}
The full transformer outputs
\begin{align}
\label{eqn:input-format-dynamical-system-after-mlp}
    \tbH = \TF_\btheta(\bH) = \begin{bmatrix}
    \Phi^\star(\barbx_1) & \Phi^\star(\barbx_2) & \dots  & \Phi^\star(\barbx_{i}) \\
    \bzero_d & \bzero_d & \dots & \bzero_d \\
    \bzero_D & \Phi^\star(\barbx_1) & \dots & \Phi^\star(\barbx_{i-1}) \\
    \bx_1 & \bx_2 & \dots & \bx_i \\
    \tbp_1 & \tbp_2 & \dots  & \tbp_i
    \end{bmatrix}.
\end{align}
where $\tbp_i,\tbp_i$ differs from $\barbp_i,\barbp_i$ only in the dimension of their zero paddings.
\end{lemma}
\begin{proof}
We first construct the first MLP layer. Consider any input token $\bh_i=[\barbx_i; \bp_i]$. Define matrices $\bW_1,\bW_2\in\R^{\Dhid\times\Dhid}$ such that (below $\pm \bu\defeq [\bu; -\bu]$)
\begin{align*}
    & \bW_1\bh_i = \begin{bmatrix}
        \pm \bB_1^\star\barbx_i \\ \pm \bx_i \\ \pm \barbx_i \\ \bzero
    \end{bmatrix}, \quad 
    \sigma\paren{ \bW_1\bh_i } = \begin{bmatrix}
        \sigma(\pm \bB_1^\star\bx_i) \\ \sigma(\pm \bx_i) \\ \sigma(\pm \barbx_i) \\ \bzero
    \end{bmatrix}, \\
    & \bW_2\sigma\paren{ \bW_1\bh_i } = \begin{bmatrix}
        \sigma(\bB_1^\star\barbx_i) - \rho\sigma(-\bB_1^\star\barbx_i) \\ \bzero
    \end{bmatrix} + \begin{bmatrix}
        -\sigma(\barbx_i) + \sigma(-\barbx_i) \\ \bzero
    \end{bmatrix} + \begin{bmatrix}
        \bzero_D \\ \sigma(\bx_i) - \sigma(-\bx_i) \\ \bzero
    \end{bmatrix}.
\end{align*}
Therefore, the MLP layer $(\bW_1,\bW_2)$ outputs
\begin{align}
\label{eqn:mlp-dynamical-system-intermediate}
    \barbh_i \defeq \brac{\MLP_{\bW_1,\bW_2}(\bH)}_i = \bh_i + \bW_2\sigma(\bW_1\bh_i) = \begin{bmatrix}
        \sigma_\rho(\bB_1^\star\barbx_i) \\ \bx_i \\ \barbp_i
    \end{bmatrix}.
\end{align}
The requirement for $\Dhid$ above is $\Dhid\ge \max\sets{2D+2(k+1)d, D+d+5}$.

The rest of the proof follows by repeating the proof of~\cref{lem:mlp} (skipping the first (MLP-Attention) block), with the following modifications:
\begin{itemize}[leftmargin=1.5em]
\item Save the $\bx_i\in\barbx_i$ location within each token, and move it into the $(2D+d+1:2D+2d)$ block in the final layer (instead of moving the label $y_i$ in~\cref{lem:mlp}); this takes the same number (at most $2$) of attention heads in every layer, same as in~\cref{lem:mlp}.
\item Append one more copying layer with a single attention head (similar as the construction in~\cref{Lem:copy-dynamic}) to copy each $\Phi^\star(\barbx_i$) to the $(D+d+1:2D+d)$ block of the next token. 
\end{itemize}
The above module has structure $(L-1)\times$(MLP-Attention), followed by a single attention layer which can be rewritten as an MLP-Attention-MLP module with identity MLP layers. Altogether, the module has an MLP-$L\times $(Attention-MLP) structure. The max number of attention heads within the above module is $5$. The required hidden dimension here is $\Dhid\ge \max\sets{kd+4, 2D+d+\max\sets{D,d}+5}$, with $\Dhid\ge \max\sets{kd, 3D+2d}+5$ being a sufficient condition. 

Combining the above two parts, a sufficient condition for $\Dhid$ is $\Dhid\ge 2(k+1)d+3D+2d+5$, as assumed in the precondition. This finishes the proof.
\end{proof}

Consider the following multi-output ridge regression problem:
\begin{align}
\label{eqn:ridge-i-dynamical-system}
    \hbWilam \defeq \argmin_{\bW\in\R^{D\times d}} \frac{1}{2(i-1)}\sum_{j=1}^{i-1} \ltwo{ \bW^\top \bx_j - \by_j}^2 + \frac{\lambda}{2}\lfro{\bW}^2.
\end{align}

\begin{theorem}[In-context multi-output ridge regression with alternative input structure]
\label{thm:ridge-multi-output}
For any $\lambda\ge 0$, $\Bx,\Bw,\By>0$ with $\kappa\defeq 1+\Bx^2/\lambda$, and $\eps<\Bx\Bw/2$, let $\Dhid\ge Dd+2(D+d)+5$, then there exists an $L$-layer transformer $\TF_\btheta$ with $M=3d$ heads and hidden dimension $\Dhid$, where
\begin{align}
\textstyle
\label{eqn:ridge-capacity-bound-dynamic}
L = \cO\paren{ \kappa\log(\Bx\Bw/(\eps)) }
\end{align}
such that the following holds. On any input matrix
\begin{align*}
    \bH = \begin{bmatrix}
    \bx_1 & \bx_2 & \dots & \bx_N \\
    \bzero_d & \bzero_d & \dots & \bzero_d \\
    \bzero_D & \bx_1 & \dots & \bx_{N-1} \\
    \bzero_d & \by_1 & \dots & \by_{N-1} \\
    \bp_1 & \bp_2 & \dots & \bp_N 
    \end{bmatrix}
\end{align*}
(where $\bx_i\in\R^D$, $\by_i\in\R^d$) such that problem~\cref{eqn:ridge-i-dynamical-system} has bounded inputs and solution: for all $i\in[N]$
\begin{align}
\label{eqn:well-conditioned-multi-output}
    \ltwo{\bx_i} \le \Bx, \quad \linfs{\by_i}\le \By, \quad \ltwoinfs{\hbWilam}\le \Bw/2,
\end{align}
$\TF_\btheta$ approximately implements the ridge regression algorithm~\cref{eqn:ridge-i-dynamical-system} at every token $i\in[N]$: The prediction $\hby_i \defeq \brac{\TF_\btheta(\bH)}_{(D+1):(D+d),i}$ satisfies 
\begin{align}
    \linf{ \hby_i - (\hbWilam)^\top\bx_i } \le \eps.
\end{align}
\end{theorem}
\begin{proof}
Observe that the multi-output ridge regression problem~\cref{eqn:ridge-i-dynamical-system} is equivalent to $d$ separable single-output ridge regression problems, one for each output dimension. Therefore, the proof follows by directly repeating the same analysis as in~\cref{thm:ridge}, with the adaptation that
\begin{itemize}[leftmargin=1.5em]
    \item Omit the copying layer since each token already admits the previous (input, label) pair;
    \item Use a $\cO(\kappa\log(\Bx\Bw/(\eps)))$-layer transformer with $3d$ heads to perform $d$ parallel ridge regression problems (each with $3$ heads), using in-context gradient descent (\cref{prop:one-step-gd}) as the internal optimization algorithm, and with slightly different input structures that can be still accommodated by using relu to implement the indicators. Further, by the precondition~\cref{eqn:well-conditioned-multi-output} and $\Dhid-2(D+d)-5\ge Dd$, we have enough empty space to store the $\bW_i^t\in\R^{D\times d}$ within the zero-paddings in $\bp_i$.
    \item Use a single-attention layer with $d$ parallel linear prediction heads (\cref{lem:linear-prediction}), one for each $j\in[d]$, to write prediction $(\hby_i)_j$ into location $(i, D+j)$ with $|(\hby_i)_j - \langle (\hbWilam)_j, \bx_i \rangle|\le \eps$. Therefore,
    \begin{align*}
        \linf{ \hby_i - (\hbWilam)^\top\bx_i } = \max_{j\in[d]} \abs{(\hby_i)_j - \<(\hbWilam)_j, \bx_i\>} \le \eps.
    \end{align*}
\end{itemize}
This finishes the proof.
\end{proof}

\subsection{Proof of Theorem~\ref{thm:dynamical-system}}
\label{app:proof-dynamical-system}

\begin{proof}[Proof of~\cref{thm:dynamical-system}]
The proof is similar as that of~\cref{thm:fixed-rep}.
The result follows directly by concatenating the following three transformer modules:

\begin{itemize}[leftmargin=1.5em]
\item The copying layer in~\cref{Lem:copy-dynamic}, which transforms the input to format~\cref{eqn:input-format-after-copy}, and thus verifies claim~\cref{eqn:copied-input-dynamical-system}.

\item The MLP representation module in~\cref{lem:mlp-dynamical-system}, which transforms~\cref{eqn:input-format-after-copy} to~\cref{eqn:input-format-dynamical-system-after-mlp}. Together with the above single attention layer, the module is now an $(L+1)$-layer transformer with $5$ heads. Claim~\cref{eqn:relued-input-dynamical-system} follows by the intermediate output~\cref{eqn:mlp-dynamical-system-intermediate} within the proof of~\cref{lem:mlp-dynamical-system}.

\item The in-context multi-output ridge regression construction in~\cref{thm:ridge-multi-output} (with inputs being $\sets{\Phi^\star(\barbx_i)}$ and labels being $\sets{\bx_{i+1}}$). This TF has $\cO\paren{\kappa\log(\Bphi\Bw/\eps)}$ layers, and $3d$ heads. It takes in input of format~\cref{eqn:input-format-dynamical-system-after-mlp}, and outputs prediction $\hby_i\defeq [\tbh_i]_{D+1:D+d}$ where $\linfs{\hby_i - 
 (\hbWiphilam)^\top\Phi^\star(\barbx_i) } \le \eps$, where $\hbWiphilam$ is the~\cref{eqn:phi-ridge-dynamical-system} predictor.
\end{itemize}
The resulting transformer has $\max\sets{3d, 5}$ heads, and the hidden dimension requirement is $$\Dhid \ge \max\sets{kd+5, 2(k+1)d+3D+2d+5, Dd+2(D+d)+5}.$$ A sufficient condition is $\Dhid=\max\sets{2(k+1),D}d+3(D+d)+5$, as assumed in the precondition. This finishes the proof.
\end{proof}

\section{Details for experiments}
\label{app:exp-details}
\textbf{Architecture and training details} We train a 12-layer decoder model in \textsc{GPT-2} family with 8 heads and hidden dimension $D_{\rm hid} = 256$, with positional encoding. We use linear read-in and read-out layer before and after the transformers respectively, both applying a same affine transform to all tokens in the sequence and are trainable. The read-in layer maps any input vector to a $D_{\rm hid}$-dimensional hidden state, and the read-out layer maps a $D_{\rm hid}$-dimensional hidden state to a 1-dimensional scalar for model \eqref{eqn:model} and to a $d$-dimensional scalar for model \eqref{eqn:dynamical-system}.

Under the in-context learning with representation setting, we first generate and fix the representation $\Phi^\star$. For a single ICL instance, We generate new coefficients $\bw$ and $N$ training examples $\set{\paren{\bx_i,y_i}}_{i\in[N]}$ and test input $\paren{\bx_{N+1},y_{N+1}}$. Before feeding into transformer, we re-format the sequence to $\bH_{\rm ICL-rep}$, as shown in equation \eqref{appeqn:exp-input-icl}. 
\begin{equation}\label{appeqn:exp-input-icl}
\bH_{\rm ICL-rep} = \left[
\bx_1, \begin{bmatrix}y_1 \\ \bzero_{d-1}\end{bmatrix}
,\ldots,
\bx_N, \begin{bmatrix}y_N \\ \bzero_{d-1}\end{bmatrix}
\right] \in \R^{d\times 2N} 
\end{equation}
We use the use the Adam optimizer with a fixed learning rate $10^{-4}$, which works well for all experiments. We train the model for $300K$ steps, where each step consists of a (fresh) minibatch with batch size $64$ for single representation experiments, except for the mixture settings in~\cref{app:mtl} where we train for $150K$ iterations, each containing $K$ batches one for each task.

Under ICL dynamic system setting, for a single ICL instance, we don't need to reformat the input sequence. We feed the original sequence $\bH_{\rm Dynamic} = [\bx_1,\ldots,\bx_N] \in \R^{d\times N}$ to transformer.

All our plots show one-standard-deviation error bars, though some of those are not too visible.

\subsection{Details for linear probing}
\label{app:probe-details}

Denote the $\ell-$th hidden state of transformers as 
\[
\bH^{(\ell)} = \begin{bmatrix}
    \bh^{x,(\ell)}_1, \bh^{y,(\ell)}_1,\ldots,
    \bh^{x,(\ell)}_N, \bh^{y,(\ell)}_{N}
\end{bmatrix}\in \R^{D_{\rm hid},2N}~~~~\text{for}~~~~\ell\in[12].
\]
Denote the probing target as $\probetarget(\set{\bx_j,y_j}_{j\in [i]}) \in \R^{d_{\rm probe}}$ for $i\in[N]$. Denote the linear probing parameter as $\bw^{x,(\ell)}$ and $\bw^{y,(\ell)}$ that belong to $\R^{D_{\rm hid} \times d_{\rm probe}}$. Denote the best linear probing model as
\begin{align*}
 \bw^{x,(\ell)}_\star &~= \argmin_{\bw^{x,(\ell)}}\E\Big[ \sum_{i=1}^N \Big\{ \paren{\bw^{x,\ell}}^\top \bh^{x,(\ell)}_i - \probetarget\Big(\set{\bx_j,y_j}_{j\in [i]}\Big) \Big\}^2 \Big]~~~~\text{and}\\
 \bw^{y,(\ell)}_\star &~= \min_{\bw^{y,(\ell)}}\E\Big[ \sum_{i=1}^N \Big\{ \paren{\bw^{y,\ell}}^\top \bh^{y,(\ell)}_i - \probetarget\Big(\set{\bx_j,y_j}_{j\in [i]}\Big) \Big\}^2 \Big].
\end{align*}
To find them, we generate $2560$ ICL input sequences with length $N$, and obtain $12$ hidden states for each input sequences. We leave $256$ sequences as test sample and use the remaining samples to estimate $\bw^{x,(\ell)}_\star$ and $\bw^{y,(\ell)}_\star$ for each $\ell$ with ordinary least squares. We use the mean squared error to measure the probe errors. In specific, define
\begin{align*}
    \text{\rm Probe Error}^{x,(\ell)}_i(g) &~= \E\Big[\Big\{ \paren{\bw^{x,\ell}_\star}^\top \bh^{x,(\ell)}_i - g\Big(\set{\bx_j,y_j}_{j\in [i]}\Big)\Big\}^2\Big]~~~~\text{with}\\
    \text{\rm Probe Error}^{x,(\ell)}(g) &~= \frac{1}{N}\sum_{i=1}^N\text{\rm Probe Error}^{x,(\ell)}_i(g),~~~\text{and}\\
    \text{\rm Probe Error}^{y,(\ell)}_i(g)&~= \E\Big[\Big\{ \paren{\bw^{x,\ell}_\star}^\top \bh^{x,(\ell)}_i - g\Big(\set{\bx_j,y_j}_{j\in [i]}\Big)\Big\}^2\Big]~~~~\text{with}\\
    \text{\rm Probe Error}^{y,(\ell)}(g) &~= \frac{1}{N}\sum_{i=1}^N\text{\rm Probe Error}^{y,(\ell)}_i(g).
\end{align*}
When $\ell = 0$, we let $\bh_i^{x,(0)} = \bh_i^{y,(0)} = \bx_i$ as a control to the probe errors in the hidden layer. We normalize each probe error with $\E[\ltwo{g(\bx,y)}^2]/d_{\rm probe}$. We use the 256 leaved-out samples to estimate these errors. We replicate the above procedure for three times and take their mean to get the final probe errors.

\subsection{Details for pasting}
\label{app:pasting-details}
From the single fixed representation settings above, we pick a trained transformer trained on the representation with $D = d = 20$ to avoid dimension mismatch between $\Phi^\star(\bx)$ and $\bx$. We choose $L=3$ and noise level $\sigma = 0.1$. 

We change the data generating procedure of $y$ from Equation \eqref{eqn:model} to
\begin{equation}\label{eqn:pasting-model}
    y_i = \<\bw,\bx_i\> + \sigma z_i,~~i\in [N],
\end{equation}
which corresponds to a linear-ICL task.
According to the results of probing Fig \ref{fig:fixed-rep-probe-z-x}, we conjecture that transformer use the first $4$ layers to recover the representation, and implement in-context learning through the $5$-th to the last layers. Therefore, we extract the $5-12$ layers as the transformer upper layers. Then paste them with three kinds of embeddings:
\begin{enumerate}[leftmargin=2em]
\item \emph{Linear} embedding $\bW\in\R^{ D_{\rm hid}\times (D+1)}$ with re-formatted input $\bH_{\rm Linear}$:
\[\bH_{\rm Linear} = \left[
\begin{bmatrix}\bx_1 \\ 0\end{bmatrix},
\begin{bmatrix}\bzero_{D} \\ y_1\end{bmatrix}
,\ldots,
\begin{bmatrix}\bx_N \\ 0\end{bmatrix},
\begin{bmatrix}\bzero_{D} \\ y_N\end{bmatrix}
\right] \in \R^{{D+1}\times 2N} 
\]
\item \emph{Linear copy} embedding $\bW\in\R^{ D_{\rm hid}\times (D+1)}$ with re-formatted input $\bH_{\rm copy}$ that copies $\bx_i$ to $y_i$ tokens in advance:
\[\bH_{\rm copy} = \left[
\begin{bmatrix}\bx_1 \\ 0\end{bmatrix},
\begin{bmatrix}\bx_1 \\ y_1\end{bmatrix}
,\ldots,
\begin{bmatrix}\bx_N \\ 0\end{bmatrix},
\begin{bmatrix}\bx_N \\ y_N\end{bmatrix}
\right] \in \R^{{D+1}\times 2N} 
\]
\item \emph{Transformer} embedding $\TF$ using the same input format $\bH_{\rm ICL-rep}$ with normal settings, as shown in \eqref{appeqn:exp-input-icl}. We extract the $4$-th layer of the \textsc{GPT-2} model, its a complete transformer block with trainable layer norm. We use a linear read-in matrix to map $\bH_{\rm ICL-rep}$ to the $D_{\rm hid}$-dimension hidden state, apply one block of transformer to it to get the \TF~embedding $\bar{\bH} = \bar{\TF}(\bH)$.
\end{enumerate}
We apply the upper layers to the three embeddings, then use the original read-out matrix to get the prediction of $\hat{y}_i$. For comparison, we also train a one-layer transformer using the input sequence $\bH_{\rm ICL-rep}$.

We use the same training objective as in \eqref{eqn:train-objective}. In the retraining process, we switch to task \eqref{eqn:pasting-model}, fix the parameters of upper layers of the transformer, and only retrain the embedding model. The training methods are exact the same with the original transformer. We also find that using a random initialized transformer block or extracting the $4$-th layer of the transformer don't make difference to the results.

\subsection{Difficulty of linear ICL with a single-layer transformer with specific input format}
\label{app:inability}

Recall the input format~\cref{eqn:input-format}:
\begin{align*}
\bH = \begin{bmatrix}
\bx_1 & \bzero & \dots & \bx_N & \bzero \\
0 & y_1 & \dots & 0 & y_{N} \\
\bp^x_1 & \bp^y_1 & \dots & \bp^x_{N} & \bp^y_{N}
\end{bmatrix}
\in \R^{\Dhid\times 2N}.
\end{align*}
Here we heuristically argue that a single attention layer alone (the only part in a single-layer transformer that handles interaction across tokens) is unlikely to achieve good linear ICL performance on input format~\cref{eqn:input-format}.

Consider a single attention head $(\bQ,\bK,\bV)$. As we wish the transformer to do ICL prediction at every token, the linear estimator $\bw_i$ used to predict $\hy_i$ is likely best stored in the $\bx_i$ token (the only token that can attend to all past data $\cD_{i-1}$ and the current input $\bx_i$). In this case, the attention layer needs to use the following (key, value) vectors to compute a good estimator $\bw_i$ from the data $\cD_{i-1}$:
\begin{align*}
    \sets{\bV\bh^x_j, \bV\bh^x_j}_{j\in[i]},~~\sets{\bV\bh^y_j, \bV\bh^y_j}_{j\in[i-1]}.
\end{align*}
However (apart from position information), $\bh^x_j$ only contains $\bx_j$, and $\bh^y_j$ only contains $y_j$. Therefore, using the normalized ReLU activation as in~\cref{app:proof-fixed-rep} \&~\ref{app:proof-dynamical-system}, it is unlikely that an attention layer can implement even simple ICL algorithms such as one step of gradient descent~\citep{von2022transformers,akyurek2022learning}:
\begin{align*}
    \bw_i = \bw_i^0 - \eta \frac{1}{i-1}\sum_{j\le i-1}\paren{\<\bw_i^0,\bx_j\> - y_j}\bx_j,
\end{align*}
which (importantly) involves term $-y_j\bx_j$ that is unlikely to be implementable by the above attention, where each attention head at each key token can observe either $\bx_j$ or $y_j$ but not both.

\section{Experiments on mixture of multiple representations}
\label{app:mtl}
We train transformers on a mixture of multiple ICL tasks, where each task admits a different representation function. This setting can be seen as a representation selection problem similar as the ``algorithm selection'' setting of~\cite{bai2023transformers}.
In specific, let $K \ge 2$ denote the number of tasks. Given $j$, let
\begin{align*}
\label{eqn:mtl-model}
y_i &~= \<\bw, \Phi_j^\star(\bx_i)\> + \sigma z_i,~~~z_i\sim \normal(0, 1),~~~i\in[N],~~~\text{where }\\
\Phi^\star_j(\bx) &~= \sigma^\star\paren{\bB_L^{\star,(j)}\sigma^{\star,(j)}\paren{\bB_{L-1}^\star\cdots\sigma^{\star,(j)}\paren{\bB_1^{\star,(j)}\bx}\cdots}},~
\bB_1^{\star,(j)}\in\R^{D\times d},~(\bB_\ell^{\star,(j)})_{\ell=2}^L\subset\R^{D\times D}.
\end{align*}
The generating distributions for $\bw$, $\set{\bx_i}_{i\in[N]}$, and $\sets{\bB_{L}^{\star,(j)}}$ are same with previous setting. We generate different $\Phi_j^\star$ for $j\in[K]$ independently. We choose $K \in \{3,6\}$, $\sigma \in \{0, 0.1,0.5\}$, $L=3$, and noise $\sigma \in\{0,0.1,0.5\}$.

At each training step, we generate $K$ independent minibatches, with the $j-$th minimatch takes the representation $\Phi^\star_j$ to generate $\set{y_i}_{i\in[N]}$. Due to multiple minibatches, we shorten the number of total training steps to $150K$. The other training details are the same with fixed single representation setting.

\begin{figure}[h]
  \centering
  \begin{minipage}{0.4\textwidth}
      \centering
      \subcaption{
      \small Risk for $K=3$}
      \label{fig:mtl-risk-K-3}
      \vspace{-.1em}
      \includegraphics[width=\linewidth]{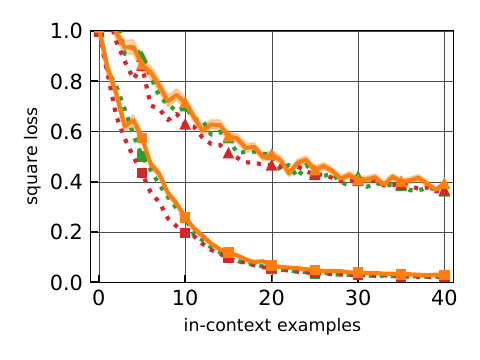}
  \end{minipage}
  \hspace{-1em}
  \begin{minipage}{0.4\textwidth}
      \centering
      \subcaption{\small Risk for $K=6$}
      \label{fig:mtl-risk-K-6}
      \vspace{-.1em}
      \includegraphics[width=\linewidth]{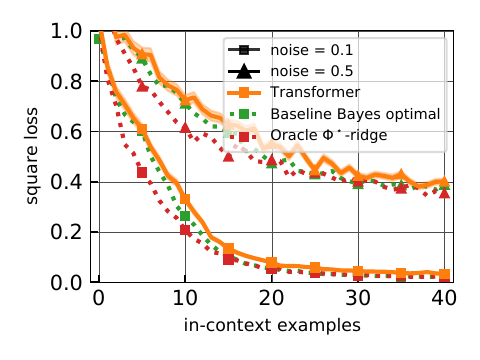}
  \end{minipage}
  \hspace{-1em}
  \vspace{-1em}
  \caption{\small ICL risks for multiple representations setting. Dotted lines plot two baseline risks. {\bf (a)} The transformer with lower risks is trained with $(K,L,D,\sigma) = (3,3,20,0.1)$. The upper one is trained with $(K,L,D,\sigma) = (3,3,20,0.5)$. {\bf (b)} The two transformers are trained with $K=6$ and same settings otherwise.}
  \label{figure:mtl}
  \vspace{-1em}
\end{figure}

\paragraph{ICL performance} 
We choose one representation $\Phi^\star_1$ from the representations that transformers are trained on. \cref{fig:mtl-risk-K-3} \& \cref{fig:mtl-risk-K-6} report the test risk. We vary $K\in\{3,6\}$ and noise level $\sigma\in\{0.1,0.5\}$. We consider two baseline models. 
\begin{enumerate}[leftmargin=2em, topsep=0pt, itemsep=0pt]
\item \emph{The Bayes optimal algorithm}: Note that the training distribution follows the Bayesian hierarchical model:
\begin{equation}\nonumber
j \sim \text{\rm Unif}([K]),~~\bx_i \sim \normal(0,\bI_d),~~\bw \sim \normal(0,\tau^2\bI_d),~~\text{and}~~
y_i~\mid~\bx_i,j,\bw \sim \normal(\<\bw,\bx_i\>,\sigma^2).
\end{equation}
This gives the Bayes optimal predictor
\begin{equation}\label{eqn:mtl-bayes-opt}
\hat{y}_i = \sum_{j=1}^K \eta_i^{(j)} \hat{y}_i^{(j)},~~~~\text{with}~~~~
(\ldots,\eta_i^{(j)},\ldots) = \textsc{Softmax}\Big\{
\Big[\ldots,\sum_{k=1}^i (y_k-\hat{y}_k^{(j)})^2/\sigma^2,\ldots\Big]\Big\}
\end{equation}
with $\hat{y}_i^{(j)}$ being ridge predictor with optimal $\lambda$ based on  $\set{\paren{\Phi^\star_j(\bx_r),y_r}}_{r \in [i-1]}$.
\item \emph{The oracle ridge algorithm:} We use the ridge predictor $\hat{y}_i^{(1)}$ based on $\set{\paren{\Phi^\star_1(\bx_r),y_r}}_{r \in [i-1]}$, which is the representation for test distribution. Note that this is an (improper) algorithm that relies on knowledge of the ground truth task.
\end{enumerate}
Comparable to those trained on single fixed representation, transformers consistently match the Bayes-optimal ridge predictor. As expected, the oracle ridge algorithm is better than transformers and the Bayes optimal algorithm and transformers. Increasing number of tasks $K$ can slightly increase this gap. Increasing the noise level has the same effect on transformers and baseline algorithms.

\begin{figure}[h]
  \centering
  \begin{minipage}{0.34\textwidth}
      \centering
      \subcaption{
      \small Probe $\Phi^\star_1(\bx_i)$}
      \label{fig:mtl-probe-z-small-noise}
      \vspace{-.1em}
      \includegraphics[width=\linewidth]{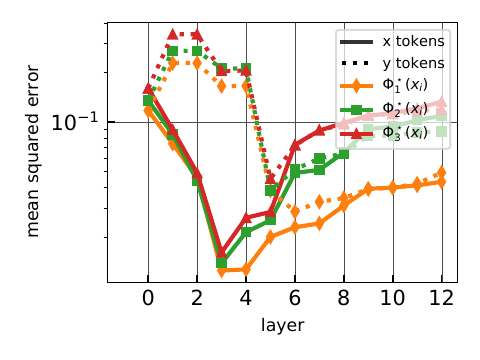}
  \end{minipage}
  \begin{minipage}{0.34\textwidth}
      \centering
      \subcaption{\small Concatenated $\hat{y}^{(j)}_i$ and $\eta_i^{(j)}$}
      \label{fig:mtl-probe-bayes-small-noise}
      \vspace{-.1em}
      \includegraphics[width=\linewidth]{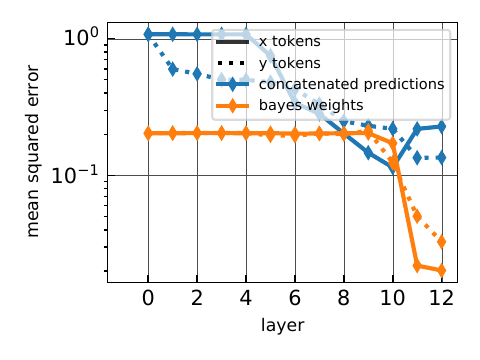}
  \end{minipage}
  \begin{minipage}{0.26\textwidth}
      \centering
      \subcaption{\small Last layer probe $\hat{y}_i^{(j)}$}
      \label{fig:mtl-probe-yhat-small-noise}
      \vspace{-.1em}
      \includegraphics[width=\linewidth]{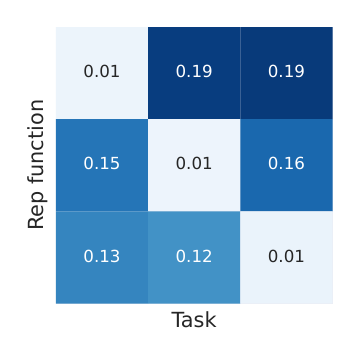}
  \end{minipage}
  \vspace{-1em}
  \caption{\small Probing errors for transformer trained with $(K,L,D,\sigma) =(3,3,20,0.1)$. Dotted lines plot probing errors on $y$ tokens. }
  \label{figure:mtl-probe-small-noise}
  \vspace{-1em}
\end{figure}

\begin{figure}[h]
  \centering
  \begin{minipage}{0.34\textwidth}
      \centering
      \subcaption{
      \small Probe $\Phi^\star_1(\bx_i)$}
      \label{fig:mtl-probe-z-big-noise}
      \vspace{-.1em}
      \includegraphics[width=\linewidth]{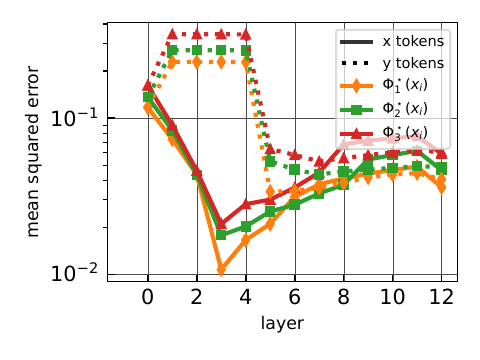}
  \end{minipage}
  \begin{minipage}{0.34\textwidth}
      \centering
      \subcaption{\small Concatenated $\hat{y}^{(j)}_i$ and $\eta_i^{(j)}$}
      \label{fig:mtl-probe-bayes-big-noise}
      \vspace{-.1em}
      \includegraphics[width=\linewidth]{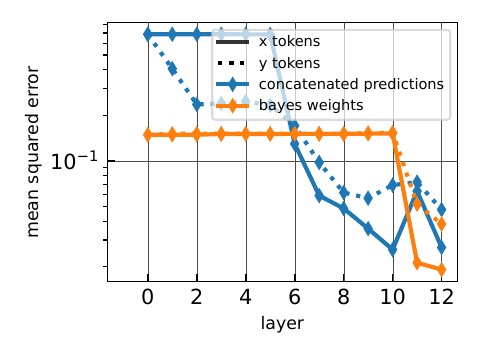}
  \end{minipage}
  \begin{minipage}{0.26\textwidth}
      \centering
      \subcaption{\small Last layer probe $\hat{y}_i^{(j)}$}
      \label{fig:mtl-probe-yhat-big-noise}
      \vspace{-.1em}
      \includegraphics[width=\linewidth]{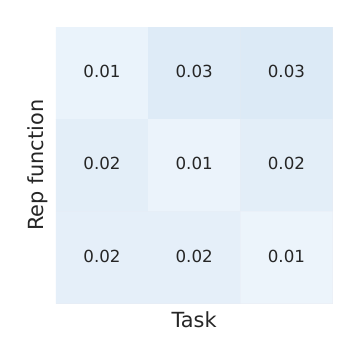}
  \end{minipage}
  \vspace{-1em}
  \caption{\small Probing errors for transformer trained with $\sigma=0.5$}
  \label{figure:mtl-probe-big-noise}
  \vspace{-1em}
\end{figure}

\paragraph{Probe setup:}
Similar to single fixed representation setting, we conduct linear probing experiments. We are wondering transformer implements the ICL-learning on representations with algorithm selections mechanism. 
We identify three sets of probing targets: $\Phi^\star(\bx_i)$, $\hat{y}^{(j)}$ and $\eta^{(j)}_i$. All of them are intermediate values to compute the Bayes optimal estimator \eqref{eqn:mtl-bayes-opt}. 
We generate different data for different probing targets:
\begin{enumerate}[leftmargin=2em, topsep=0pt, itemsep=0pt]
    \item To probe $\Phi^\star(\bx_i)$ and $\hat{y}^{(j)}$ for each $j$, we choose one representation from the representations that transformers are trained on, then train and test our linear probing model. This is consistent with the training and testing methods for probing transformers trained on a single representation.
    \item To probe choose concatenated probing targets $\bY_i = [\hat{y}_i^{(1)},\ldots,\hat{y}_i^{(K)}]$ and $\bB_i = [\eta_i^{(1)},\ldots,\eta_i^{(K)}]$, we generate $2560$ in-context sequences for each representation, and obtain $2560\times K$ samples together. We use ordinary linear square on $2560\times K - 256$ samples to get the linear probing models. Then test them on the remaining $256$ samples to get the probing errors. We also repeat this process for three times and take means to get the final probing errors.
\end{enumerate}

\paragraph{Probe representations:}
Take the transformer trained on $K=3$ mixture representations with noise level $\sigma\in\{0.1,0.5\}$. \cref{figure:mtl-probe-small-noise} show the probing errors for $\sigma=0.1$:
\cref{fig:mtl-probe-z-small-noise} reports the errors of probing $\Phi^\star_j$ $j\in[3]$, with probing models trained on task $\Phi^\star_1$. Echoing the results for transformers trained on single representation, the probing errors for each representations decrease through lower layers and increase through upper layers on $\bx$ tokens. The probing errors on $y$ tokens drop after $\bx$ tokens, which suggests a copy mechanism.
Surprisingly, on $x$-tokens, the probing errors for all representations attain their minimum at the $3$-th layer, with transformers trained on single representation achieving their minimum on $4$-th layer (compare with \cref{fig:fixed-rep-probe-z-x}).

More importantly, for both $\bx$ and $y$ tokens, the probing errors for each representation are similar through lower layers, but the probing errors for the true representation $\Phi^\star_1$ become the lowest through the upper layers. The gap between the probing errors increases. At the last layer, the probing error for the other representations go up to match the initial input. 

\paragraph{Probe intermediate values for computing Bayes optimal predictor:}
\cref{fig:mtl-probe-bayes-small-noise} shows the probing errors for concatenated ridge predictors $\hat{y}_i^{(j)}$ and Bayes weights $\eta_i^{(j)}$, i.e., $\bY_i$ and $\bB_i$. The probing errors for $\bY_i$ start dropping at the $4-$th layer, which suggest that transformer are implementing ICL using each representations. Probing errors for $\bB_i$ have a sudden drop at the $10-$th layer. \cref{fig:mtl-probe-yhat-small-noise} shows the probing errors for probing $\hat{y}_i^{(j)}$. At $(j,k)$-th cell, we show the probing error of $\hat{y}_i^{(j)}$ with probing models trained on $\Phi^\star_k$ at the $\bx$ tokens of the last layer. The diagonal elements dominant. The results combined together suggest the possibility that transformer compute in-context learning with three representations and implement algorithm selections at the $10-$th layer to drop some predictions.

In comparison, \cref{figure:mtl-probe-big-noise} shows results of probing the same targets for transformer under $\sigma=0.5$. \cref{fig:mtl-probe-z-big-noise} differs with \cref{fig:mtl-probe-bayes-small-noise} at upper layers, where probing errors for different representations don't have significant gaps. \cref{fig:mtl-probe-bayes-big-noise} is close to \cref{fig:mtl-probe-bayes-small-noise}, also suggesting the algorithm selection mechanism. \cref{fig:mtl-probe-yhat-big-noise} shows that the last layer encodes the information of all ridge predictors $\set{\hat y_i^{(j)}}$, which is drastically different from the results in \cref{fig:mtl-probe-yhat-small-noise}.

\paragraph{Conjecture on two different algorithm selection mechanisms:}
Based on the empirical findings, we conjecture two possible mechanisms of algorithm selection in transformer: (1) For small noise level data, transformers implement ``concurrent-ICL algorithm selection'', which means they concurrently implement ICL with algorithm selection, then stop implementing the full ICL procedure for algorithms that not are not likely to have good performance. (2) For large noise level data, transformers ``post-ICL algorithm selection'', which means they first implement ICL using each algorithm, then select and output the best one. However, we need further experimental and theoretical to inspect this conjecture.

\section{Ablations}
\label{app:ablations}

\subsection{Supervised learning with representation}

\paragraph{Probing results along training trajectory}
\cref{fig:fix-train-2000}, \cref{fig:fix-train-5000}, and \cref{fig:fix-train-10000} show the probing error for $\Phi^\star(\bx_i)$ at $\bx$ and $y$ tokens and $\hat{y}^{\Phi^\star \text{\rm ridge}}$ at $\bx$ tokens. As expected, all probe errors reduce through training steps, showing that the progress of learning $\Phi^\star$ is consistent with the progress of the training loss. At the 2000 training steps, transformer cannot recover the representation. At the 5000 training steps, the transformer starts memorizing the representation, starting showing differences between lower and upper layers. From 5000 training steps to 10000, the trend of probe errors varying with layers remains the same.

\begin{figure}[ht]
  \centering
  \begin{minipage}{0.33\textwidth}
      \centering
      \subcaption{
      \small 2000 training steps}
      \label{fig:fix-train-2000}
      \vspace{-.1em}
      \includegraphics[width=\linewidth]{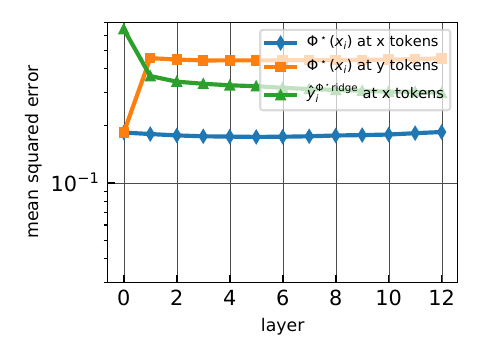}
  \end{minipage}
  \hspace{-1em}
  \begin{minipage}{0.33\textwidth}
      \centering
      \subcaption{
      \small 5000 training steps}
      \label{fig:fix-train-5000}
      \vspace{-.1em}
      \includegraphics[width=\linewidth]{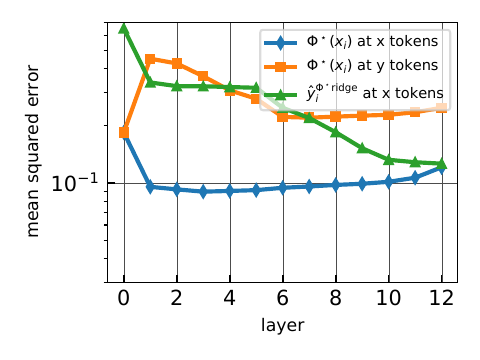}
  \end{minipage}
  \hspace{-1em}
  \begin{minipage}{0.33\textwidth}
      \centering
      \subcaption{
      \small 10000 training steps}
      \label{fig:fix-train-10000}
      \vspace{-.1em}
      \includegraphics[width=\linewidth]{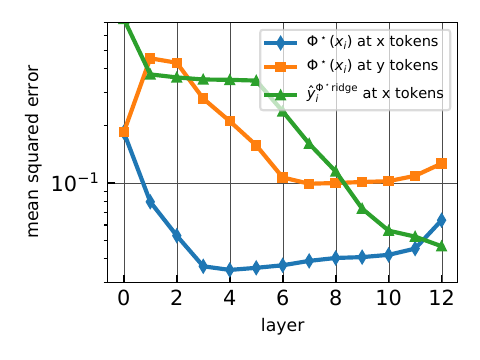}
  \end{minipage}
  \vspace{-1em}
  \caption{\small Probing errors for transformer trained after 2000, 5000, and 10000 steps. All three plots are for the training run on $(L,D,\sigma) = (2,10,0.1)$.}
  \label{figure:train}
  \vspace{-1em}
\end{figure}

\paragraph{Additional results for probing and pasting}
\cref{fig:per-token} plots the same probing errors as in~\cref{fig:fixed-rep-probe-z-x} with $(L,D,\sigma)=(3,20,0.1)$ (the green line there), except that we separate the errors of the first 4 tokens with the rest (token 5-41), but the probing training remains the same (pooled across all tokens). We observe that lower layers compute the representation in pretty much the same ways, though later layers forget the representations more for the beginning tokens (1-4) than the rest tokens.

\cref{fig:paste-ablations} plots the same pasting experiment as in~\cref{fig:paste-risk}, except that for noise level $\sigma=0.5$ as opposed to $\sigma=0.1$ therein. The message is mostly the same as in~\cref{fig:paste-risk}.

\begin{figure}[ht]
  \centering
  \begin{minipage}{0.5\textwidth}
    \centering
    \subcaption{\small Probe errors per token}
    \label{fig:per-token}
    \includegraphics[width=\linewidth]{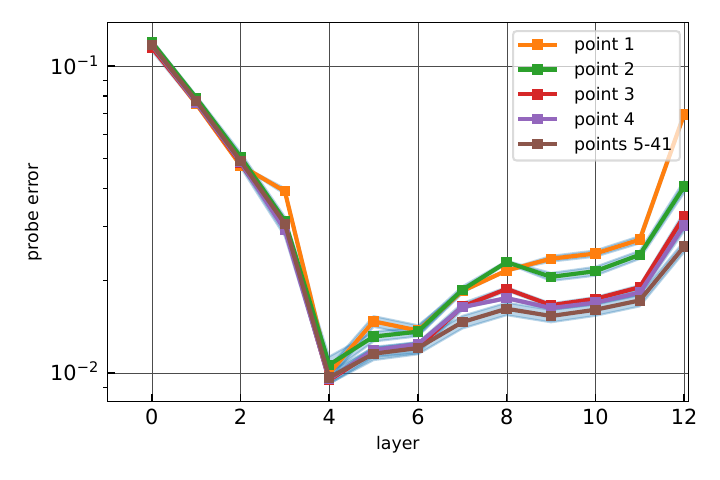}
  \end{minipage}
  \begin{minipage}{0.45\textwidth}
    \centering
    \subcaption{\small Pasting experiment}
    \label{fig:paste-ablations}
    \includegraphics[width=\linewidth]{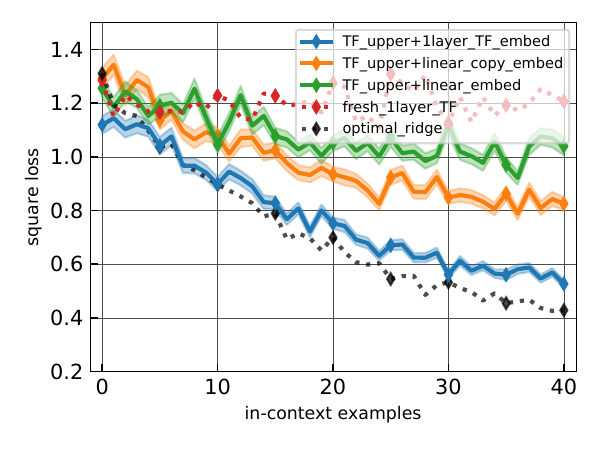}
  \end{minipage}
  \vspace{-1em}
  \caption{\small {\bf (a)} Probing errors of $\Phi^\star(\bx_i)$ in $\bx_i$ tokens evaluated per-token. {\bf (b)} Pasting results for the upper module of a trained transformer in setting $(L, D, \sigma)=(3, 20, 0.5)$.
  }
  \label{figure:combined-ablations}
  \vspace{-1em}
\end{figure}

\subsection{Dynamical systems}
\label{app:ablations-dynamical-system}

\paragraph{Risk} \cref{figure:dynamics-risk} gives ablation studies for the ICL risk in the dynamical systems setting in~\cref{sec:dynamical-system}. In all settings, the trained transformer achieves nearly Bayes-optimal risk. Note that the noise appears to have a larger effect than the hidden dimension, or the number of input tokens.

\begin{figure}[ht]
  \centering
  \begin{minipage}{0.34\textwidth}
      \centering
      \subcaption{\small Varying noise level}
      \label{fig:dynamics-risk-noise}
      \vspace{-.2em}
      \includegraphics[width=\linewidth]{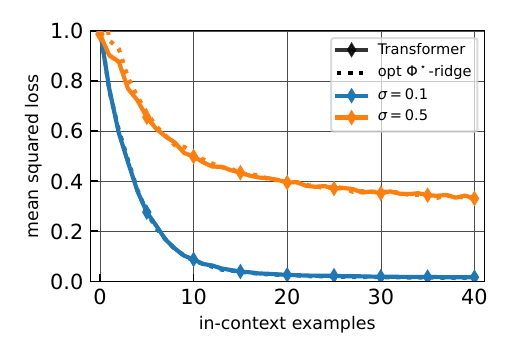}
  \end{minipage}
  \hspace{-1em}
  \begin{minipage}{0.32\textwidth}
      \centering
      \subcaption{\small Varying rep hidden dim.}
      \label{fig:dynamics-risk-D}
      \vspace{-.2em}
      \includegraphics[width=\linewidth]{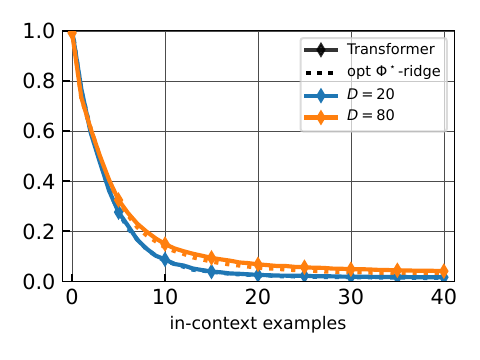}
  \end{minipage}
  \hspace{-1em}
  \begin{minipage}{0.32\textwidth}
      \centering
      \subcaption{\small Varying number of input tokens}
      \label{fig:dynamics-risk-k}
      \vspace{-.2em}
      \includegraphics[width=\linewidth]{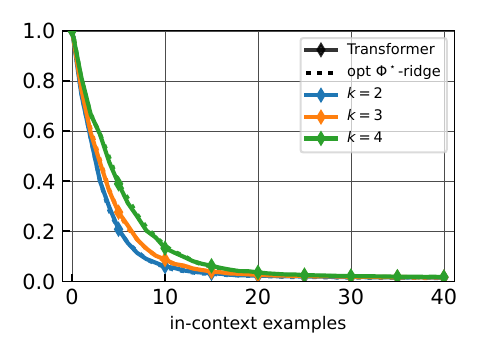}
  \end{minipage}
  \vspace{-1em}
  \caption{\small Ablation studies for the risk for Risk for fixed rep setting. Each plot modifies a single problem parameter from the base setting $(k,L,D,\sigma)=(3,2,20,0.1)$.
  }
  \label{figure:dynamics-risk}
  \vspace{-1em}
\end{figure}

\paragraph{Probing} \cref{fig:dynamics-copy1-ablations} \&~\ref{fig:dynamics-copy2-ablations} gives ablation studies for the probing errors in the dynamical systems setting in~\cref{sec:dynamical-system}, with $D=20$ instead of $D=80$ as in~\cref{fig:dynamics-copy1} \&~\ref{fig:dynamics-copy2}. The message is largely similar except that in~\cref{fig:dynamics-copy1-ablations}, all past inputs and intermediate steps in $\Phi^\star(\barbx_i)$ are simultaneously best implemented after layer 4.

\begin{figure}[ht]
  \centering
  \begin{minipage}{0.45\textwidth}
      \centering
      \subcaption{\small Probe past inputs at $\bx_i$ tokens}
      \label{fig:dynamics-copy1-ablations}
      \vspace{-.1em}
      \includegraphics[width=\linewidth]{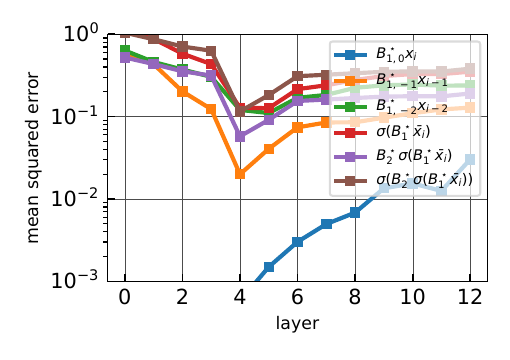}
  \end{minipage}
  \begin{minipage}{0.45\textwidth}
      \centering
      \subcaption{\small Probe $\Phi^\star(\bx_{i-j})$ at $\bx_i$ tokens}
      \label{fig:dynamics-copy2-ablations}
      \vspace{.4em}
      \includegraphics[width=\linewidth]{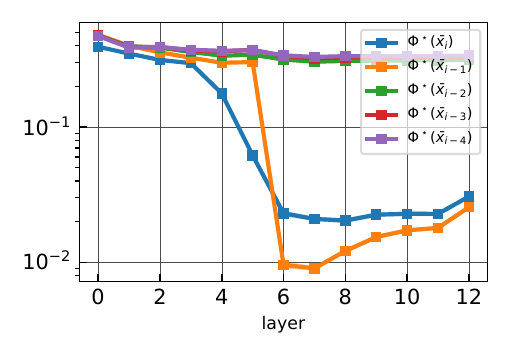}
  \end{minipage}
  \vspace{-1em}
  \caption{\small Ablation study for the probing errors in the dynamics setting. Here $(k, L, D, \sigma)=(3, 2, 20, 0.5)$, different from~\cref{figure:dynamics} where $D=80$.
  }
  \label{figure:dynamics-ablations}
  \vspace{-1em}
\end{figure}

\end{document}